%% file: main.tex
\documentclass{article}

\usepackage[nonatbib, final]{neurips_2023}

\usepackage[utf8]{inputenc} 
\usepackage[T1]{fontenc}    
\usepackage{hyperref}       
\usepackage{url}            
\usepackage{booktabs}       
\usepackage{amsfonts}       
\usepackage{nicefrac}       
\usepackage{microtype}      

\usepackage[american]{babel}
\usepackage{mathtools, amsmath, amsfonts, amssymb, amsthm, bbold}
\usepackage[table, svgnames, dvipsnames]{xcolor}
\usepackage{makecell, cellspace, caption}

\usepackage[square,numbers]{natbib}
\usepackage{mathtools} 
\usepackage{tikz} 
\usepackage{caption}
\usepackage{subcaption}

\usepackage{diagbox}

\newcommand{\ourcdf}{\textit{RID}}
\newcommand{\RLDCDF}{\textit{RLD}}
\newcommand{\ourvi}{\textit{RIV}}
\newcommand{\RID}{\textit{RID}}
\newcommand{\RLD}{\textit{RLV}}
\newcommand{\LD}{\textit{LD}^*}

\newcommand{\Dn}{\mathcal{D}^{(n)}}

\newtheorem{assumption}{Assumption}
\newtheorem{theorem}{Theorem}
\newtheorem{lemma}{Lemma}
\newtheorem{corollary}{Corollary}
\newtheorem{proposition}{Proposition}

\title{The Rashomon Importance Distribution: Getting RID of Unstable, Single Model-based Variable Importance}

\author{%
  Jon Donnelly*\\
  Department of Computer Science\\
  Duke University\\
  Durham, NC 27708 \\
  \texttt{jon.donnelly@duke.edu} \\
  \And
  Srikar Katta*\\
  Department of Computer Science\\
  Duke University\\
  Durham, NC 27708 \\
  \texttt{srikar.katta@duke.edu} \\
  \And
  Cynthia Rudin \\
  Department of Computer Science\\
  Duke University\\
  Durham, NC 27708 \\
  \texttt{cynthia.rudin@duke.edu} \\
  \And
  Edward P. Browne \\
  Department of Medicine\\
  University of North Carolina at Chapel Hill\\
  Chapel Hill, NC 27599\\
  \texttt{epbrowne@email.unc.edu} \\
}

\makeatletter
\def\blfootnote{\xdef\@thefnmark{}\@footnotetext}
\makeatother

\begin{document}

\blfootnote{*Jon Donnelly and Srikar Katta contributed equally to this work.}
\maketitle
\begin{abstract}
  Quantifying variable importance is essential for answering high-stakes questions in fields like genetics, public policy, and medicine. Current methods generally calculate variable importance for a given model trained on a given dataset. However, for a given dataset, there may be many models that explain the target outcome equally well; without accounting for all possible explanations, different researchers may arrive at many conflicting yet equally valid conclusions given the same data. Additionally, even when accounting for all possible explanations for a given dataset, these insights may not generalize because not all good explanations are stable across reasonable data perturbations. We propose a new variable importance framework that quantifies the importance of a variable across the set of all good models and is stable across the data distribution. Our framework is extremely flexible and can be integrated with most existing model classes and global variable importance metrics. We demonstrate through experiments that our framework recovers variable importance rankings for complex simulation setups where other methods fail. Further, we show that our framework accurately estimates the \textit{true importance} of a variable for the underlying data distribution. We provide theoretical guarantees on the consistency and finite sample error rates for our estimator. Finally, we demonstrate its utility with a real-world case study exploring which genes are important for predicting HIV load in persons with HIV, highlighting an important gene that has not previously been studied in connection with HIV. 
  
\end{abstract}


\section{Introduction }\label{sec:intro}
\input{paper_files/introduction3.tex}
\section{Related Work}\label{sec:related_work}
\input{paper_files/related_works.tex}

\section{Methods} \label{sec:methods}
\input{paper_files/methods_no_jrid.tex}
\section{Experiments With Known Data Generation Processes} \label{sec:experiments}
\input{paper_files/synthetic_experiments.tex} 
\section{Case Study}

\input{paper_files/case_study}
\section{Conclusion and Limitations} \label{sec:conclusion}

\input{paper_files/conclusion.tex}

\section{Acknowledgements}
We gratefully acknowledge support from NIH/NIDA R01DA054994, NIH/NIAID R01AI143381, DOE DE-SC0023194, NSF IIS-2147061, and NSF IIS-2130250.
\small


\makeatletter
\@addtoreset{theorem}{section}
\makeatother
\makeatletter
\@addtoreset{assumption}{section}
\makeatother
\newpage
\input{supplement_files/supplement}

\bibliography{paper_files/references}

\end{document}

%% file: paper_files/introduction3.tex
Variable importance analysis enables researchers to gain insight into a domain or a model. Scientists are often interested in understanding causal relationships between variables, but running randomized experiments is time-consuming and expensive. Given an observational dataset, we can use global variable importance measures to check if there is a predictive relationship between two variables. It is particularly important in high stakes real world domains such as genetics  \citep{wang2020initial, novakovsky2022obtaining}, finance \citep{rudin2019we}, and criminal justice \citep{FisherRuDo19, propublica} where randomized controlled trials are impractical or unethical. 
Variable importance would ideally be measured as the importance of each variable to the data generating process. However, the data generating process is never known in practice, so prior work generally draws insight by analyzing variable importance for a surrogate model, treating that model and its variable importance as truth.

This approach can be misleading because there may be many good models for a given dataset -- a phenomenon referred to as the Rashomon effect \citep{breiman2001statistical, semenova2022existence} --- and variables that are important for one good model on a given dataset are \textit{not} necessarily important for others. As such, any insights drawn from a single model need not reflect the underlying data distribution or even the consensus among good models.
Recently, researchers have sought to overcome the Rashomon effect by computing \textit{Rashomon sets}, the set of all good (i.e., low loss) models for a given dataset \citep{FisherRuDo19, dong2020exploring}. However, \textit{the set of all good models is not stable across reasonable perturbations (e.g., bootstrap or jackknife) of a single dataset}, with stability defined as in \citep{yu2013stability}. This concept of stability is one of the three pillars of veridical data science \citep{yu2020veridicalpnas, duncan2022veridicalflow}. Note that there is wide agreement on the intuition behind stability, but not its quantification \citep{kalousis2005stability, nogueira2017stability}. As such, in line with other stability research, we do not subscribe to a formal definition and treat stability as a general notion \citep{kalousis2005stability, nogueira2017stability, yu2013stability, yu2020veridicalpnas}. In order to ensure trustworthy analyses, variable importance measures must account for both the Rashomon effect and stability.

Figure \ref{fig:mcr_unstable} provides a demonstration of this problem: across 500 bootstrap replicates from \textit{the same} data set, the Rashomon set varies wildly -- ranging from ten models to over \textit{ten thousand} --- suggesting that we should account for its instability in any computed statistics. This instability is further highlighted when considering the Model Class Reliance (MCR) variable importance, which is the range of model reliance (i.e., variable importance) values across the Rashomon set for the given dataset \citep{FisherRuDo19} (we define MCR and the Rashomon set more rigorously in Sections \ref{sec:related_work} and \ref{sec:methods} respectively).
In particular, for variable $X_2$, one interval --- ranging from -0.1 to 0.33 --- suggests that there exist good models that do not depend on this variable at all (0 indicates the variable is not important); 
on the other hand, another MCR from a bootstrapped dataset ranges from 0.33 to 0.36, suggesting that this variable is essential to all good models. Because of this instability, different researchers may draw very different conclusions about the same data distribution even when using the same method.

\begin{figure}
     \centering
     \begin{subfigure}{0.48\textwidth}
         \includegraphics[width=\textwidth]{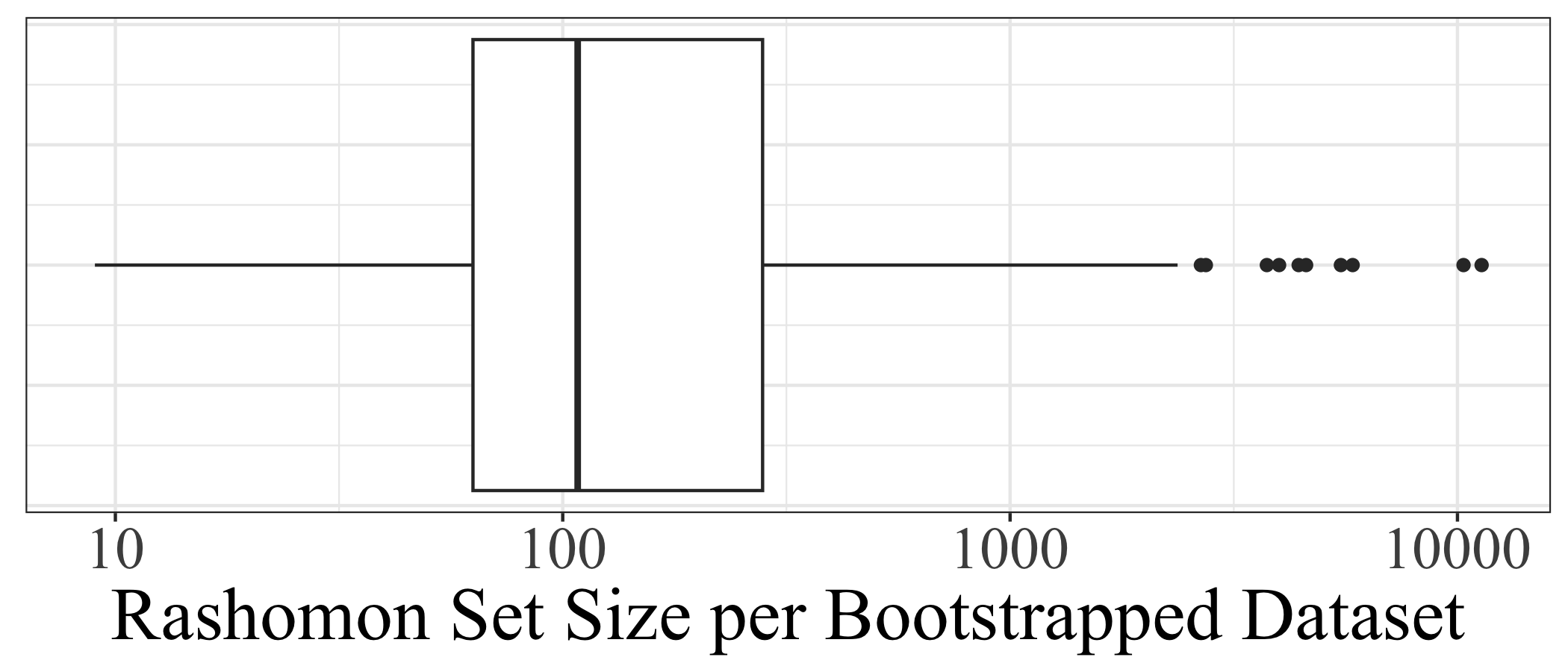}
         \caption{The number of models in each Rashomon set}
     \end{subfigure}
     \begin{subfigure}{0.48\textwidth}
         \includegraphics[width=\textwidth]{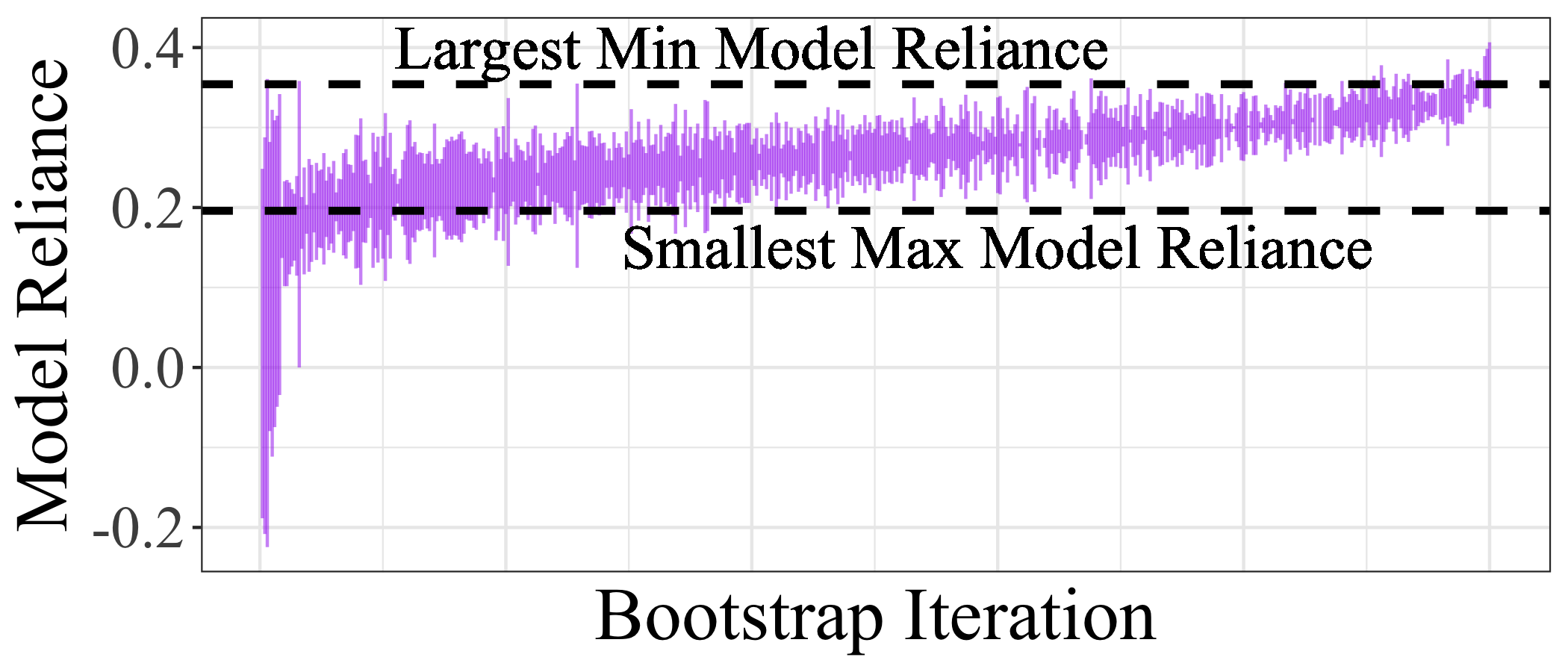}
         \caption{Model reliance range across each Rashomon set}
     \end{subfigure}
     \caption{Statistics of Rashomon sets computed across 500 bootstrap replicates of a given dataset sampled from the Monk 3 data generation process \citep{thrun1991monk}. The original dataset consisted of 124 observations, and the Rashomon set was calculated using its definition in Equation \ref{eqn:Rset}, with parameters specified in Section D of the supplement. The Rashomon set size is the number of models with loss below a threshold. Model reliance is a measure of variable importance for a single variable --- in this case, $X_2$ --- and Model Class Reliance (MCR) is its range over the Rashomon set.  Both the Rashomon set size and model class reliance are unstable across bootstrap iterations.}
    \label{fig:mcr_unstable}
\end{figure}

\begin{figure}
     \centering
     \begin{subfigure}{0.11\textwidth}
         \includegraphics[width=\textwidth]{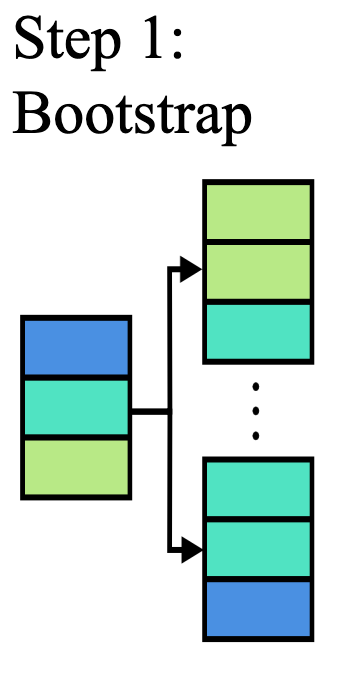}
     \end{subfigure}
     \begin{subfigure}{0.21\textwidth}
         \includegraphics[width=\textwidth]{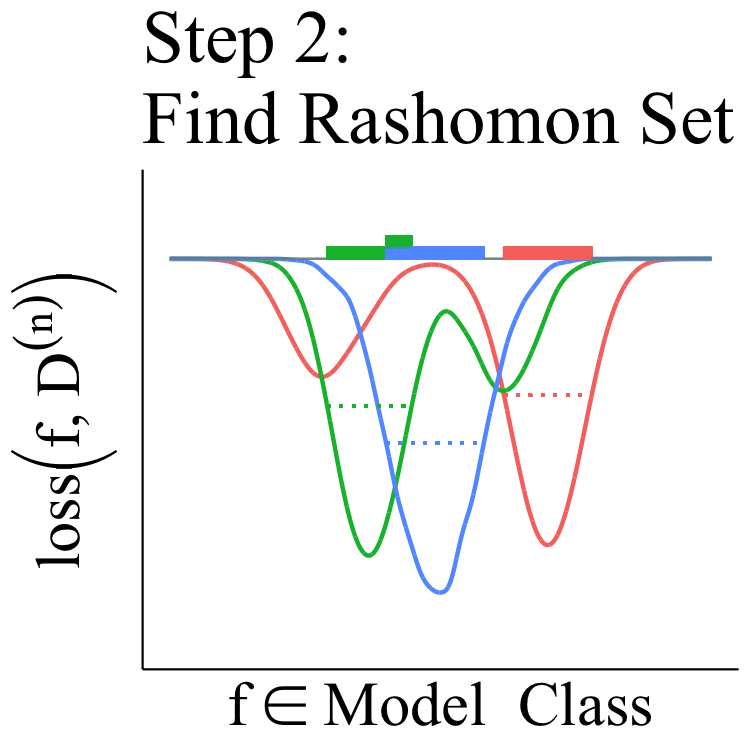}
     \end{subfigure}
     \begin{subfigure}{0.21\textwidth}
         \includegraphics[width=\textwidth]{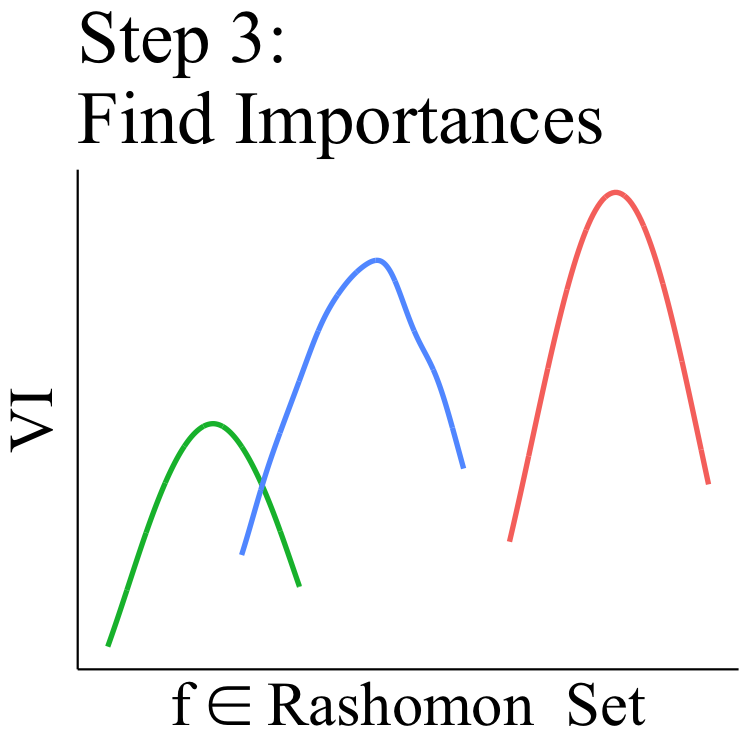}
     \end{subfigure}
     \begin{subfigure}{0.21\textwidth}
         \includegraphics[width=\textwidth]{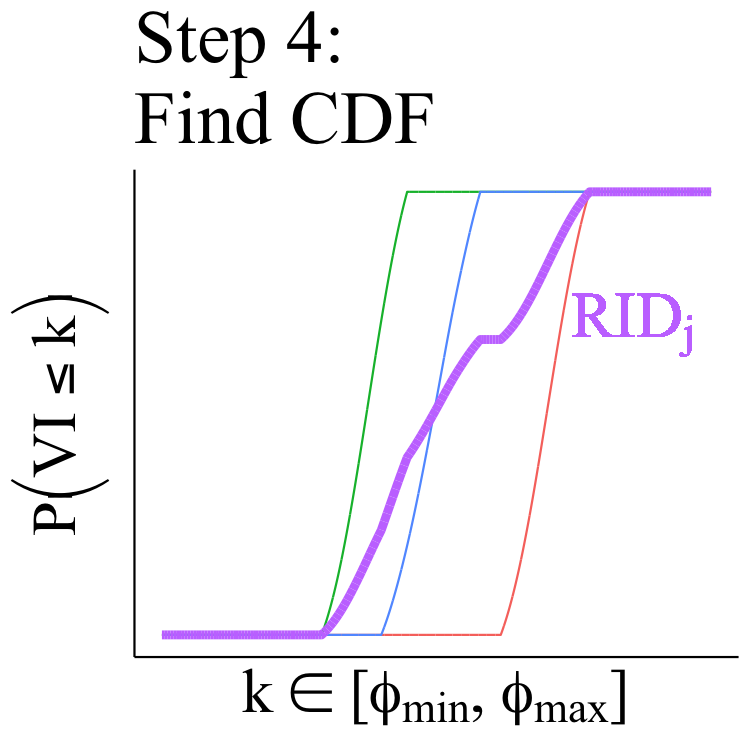}
     \end{subfigure}
     \begin{subfigure}{0.21\textwidth}
         \includegraphics[width=\textwidth]{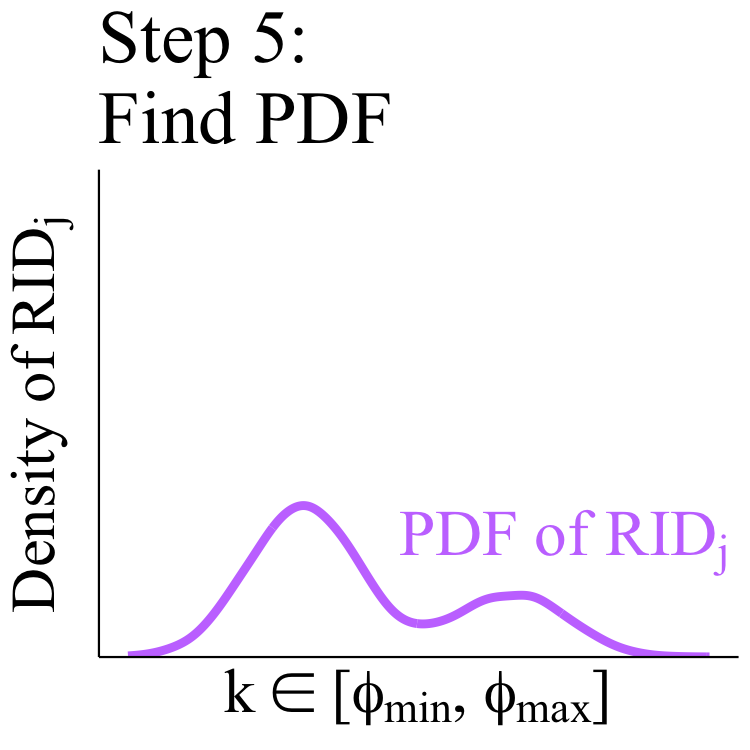}
     \end{subfigure}
        \caption{An overview of our framework. \textbf{Step 1:} We bootstrap multiple datasets from the original. \textbf{Step 2:} We show the loss values over the model class for each bootstrapped dataset, differentiated by color. The dotted line marks the Rashomon threshold; all models whose loss is under the threshold are in the Rashomon set for that bootstrapped dataset. On top, we highlight the number of bootstrapped datasets for which the corresponding model is in the Rashomon set. \textbf{Step 3:} We then compute the distribution of model reliance (variable importance -- VI) values for variable $j$ across the Rashomon set for each bootstrapped dataset. \textbf{Step 4:} We then average the corresponding CDF across bootstrap replicates into a single CDF (in purple). \textbf{Step 5:} Using the CDF, we compute the marginal distribution (PDF) of variable importance for variable $j$ across the Rashomon sets of bootstrapped datasets.}
        \label{fig:rashomonPipeline}
\end{figure}

In this work, we present a framework unifying concepts from classical nonparametric estimation with recent developments on Rashomon sets to overcome the limitations of traditional variable importance measurements. We propose a stable, model- and variable-importance-metric-agnostic estimand that quantifies variable importance across all good models for the empirical data distribution and a corresponding bootstrap-style estimation strategy. 
Our method creates a cumulative density function (CDF) for variable importance over all variables via the framework shown in Figure \ref{fig:rashomonPipeline}. Using the CDF, we can compute a variety of statistics (e.g., expected variable importance, interquartile range, and credible regions) that can summarize the variable importance distribution.

The rest of this work is structured as follows. After more formally introducing our variable importance framework, we theoretically guarantee the convergence of our estimation strategy and derive error bounds. We also demonstrate experimentally that our estimand captures the true variable importance for the data generating process more accurately than previous work. Additionally, we illustrate the generalizability of our variable importance metric by analyzing the reproducibility of our results given new datasets from the same data generation process. Lastly, we use our method to analyze which transcripts and chromatin patterns in human T cells are associated with high expression of HIV RNA. Our results suggest an unexplored link between the LINC00486 gene and HIV load. 


Code is available at \href{https://github.com/jdonnelly36/Rashomon_Importance_Distribution}{https://github.com/jdonnelly36/Rashomon\_Importance\_Distribution}.

%% file: paper_files/related_works.tex
The key difference between our work and most others is the way it incorporates model uncertainty, also called the Rashomon effect \citep{breiman2001statistical}. The Rashomon effect is the phenomenon in which many different models explain a dataset equally well. It has been documented in high stakes domains including healthcare, finance, and recidivism prediction \citep{d2020underspecification, marx2020predictive, FisherRuDo19}. The Rashomon effect has been leveraged to create uncertainty sets for robust optimization \citep{tulabandhula2014robust}, to perform responsible statistical inference \citep{coker2021theory}, and to gauge whether simple yet accurate models exist for a given dataset \citep{semenova2022existence}.
One barrier to studying the Rashomon effect is the fact that \textit{Rashomon sets} are computationally hard to calculate for non-trivial model classes. Only within the past year has code been made available to solve for (and store) the full Rashomon set for any nonlinear function class -- that of decision trees \citep{xin2022exploring}. This work enables us to revisit the study of variable importance with a new lens.

A classical way to determine the importance of a variable is to leave it out and see if the loss changes. This is called algorithmic reliance \citep{FisherRuDo19} or leave one covariate out (LOCO) inference 
\citep{lei2018distribution, rinaldo2019bootstrapping}. The problem with these approaches is that the performance of the model produced by an algorithm will not change if there exist other variables correlated with the variable of interest.

Model reliance (MR) methods capture the global variable importance (VI) of a given feature \textit{for a specific model} \citep{FisherRuDo19}. (Note that MR is limited to refer to permutation importance in \cite{FisherRuDo19}, while we use the term MR to refer to any metric capturing global variable importance of a given feature and model. We use VI and MR interchangably when the relevant model is clear from context.) 
Several methods for measuring the MR of a model from a specific model class exist, including the variable importance measure from random forest which uses out-of-bag samples \citep{breiman2001statistical} and Lasso regression coefficients \citep{hastie2009elements}. \citet{lundberg2018consistent} introduce a way of measuring MR in tree ensembles using SHAP \citep{SHAP}. \citet{williamson2021general} develop MR based on the change in performance between the optimal model and the optimal model using a subset of features. 

In addition to the metrics tied to a specific model class, many MR methods can be applied to models \textit{from any model class}. Traditional correlation measures \citep{hastie2009elements} can measure the linear relationship (Pearson correlation) or general monotonic relationship (Spearman correlation) between a feature and predicted outcomes for a model from any model class. Permutation model reliance, as discussed by \cite{altmann2010permutation, FisherRuDo19, hooker2021unrestricted}, describes how much worse a model performs when the values of a given feature are permuted such that the feature becomes uninformative. Shapley-based measures of MR, such as those of \cite{williamson2020efficient, SHAP}, calculate the average marginal contribution of each feature to a model's predictions. A complete overview of the variable importance literature is beyond the scope of this work; for a more thorough review, see, for example,
\cite{arrieta2020explainable, minh2022explainable}. Rather than calculating the importance of a variable for a single model, our framework finds the importance of a variable for all models within a Rashomon set, although our framework is applicable to \textit{all} of these model reliance metrics.


In contrast, model class reliance (MCR) methods describe how much a \textit{class of models} (e.g., decision trees) relies on a variable. \citet{FisherRuDo19} uses the Rashomon set to provide bounds on the possible \textit{range} of model reliance for good models of a given class. \citet{NEURIPS2020_fd512441} analytically find the range of model reliance for the model class of random forests. \citet{zhang2020floodgate} introduce a way to compute confidence bounds for a specific variable importance metric over arbitrary models, which \citet{massimo2022floodgate} extend so that it is applicable to a  broad class of surrogate models in pursuit of computational efficiency. These methods report MCR as a range, which gives no estimate of variable importance -- only a range of what values are possible. In contrast, \citet{dong2020exploring} compute and visualize the variable importance for every member of a given Rashomon set in projected spaces, calculating a set of points; however, these methods have no guarantees of stability to reasonable data perturbations.
In contrast, our framework overcomes these finite sample biases, supporting stronger conclusions about the underlying data distribution.

Related to our work from the stability perspective, \citet{duncan2022veridicalflow} developed a software package to evaluate the stability of permutation variable importance in random forest methods; we perform a similar exercise to demonstrate that current variable importance metrics computed for the Rashomon set are not stable. Additionally, \citet{basu2018iterative} introduced iterative random forests by iteratively reweighting trees and bootstrapping to find \textit{stable} higher-order interactions from random forests. Further, theoretical results have demonstrated that bootstrapping stabilizes many machine learning algorithms and reduces the variance of statistics \citep{grandvalet2006stability, buhlmann2002analyzing}. We also take advantage of bootstrapping's flexibility and properties to ensure stability for our variable importance.

%% file: paper_files/methods_no_jrid.tex
\subsection{Definitions and Estimands} \label{sec:defEstimands}
Let $\mathcal{D}^{(n)} = \{(X_i, Y_i)\}_{i=1}^n$ denote a dataset of $n$  independent and identically distributed tuples, where $Y_i\in \mathbb{R}$ denotes some outcome of interest and $X_i\in \mathbb{R}^p$ denotes a vector of $p$ covariates. Let $g^*$ represent the data generating process (DGP) producing $\Dn$. Let $f \in \mathcal{F}$ be a model in a model class (e.g., a tree in the set of all possible sparse decision trees), and let $\phi_j\left(f, \mathcal{D}^{(n)}\right)$ denote a function that measures the importance of variable $j$ for a model $f$ over a dataset $\Dn.$ This can be any of the functions described earlier (e.g., permutation importance, SHAP). Our framework is flexible with respect to the user's choice of $\phi_j$ and enables practitioners to use the variable importance metric best suited for their purpose; for instance, conditional model reliance \cite{FisherRuDo19} is best-suited to measure only the unique information carried by the variable (that cannot be constructed using other variables), whereas other metrics like subtractive model reliance consider the unconditional importance of the variable. Our framework is easily integrable with either of these. We only assume that the variable importance function $\phi$ has a bounded range, which holds for a wide class of metrics like SHAP \cite{SHAP}, permutation model reliance, and conditional model reliance. Finally, let $\ell(f, \Dn; \lambda)$ represent a loss function given $f, \Dn,$ and loss hyperparameters $\lambda$ (e.g., regularization).  We assume that our loss function is bounded above and below, which is true for common loss functions like 0-1 classification loss, as well as for differentiable loss functions with covariates from a bounded domain.

In an ideal setting, we would measure variable importance using $g^*$ and the whole population, but this is impossible because $g^*$ is unknown and data is finite. 
In practice, scientists instead use the empirical loss minimizer for a specific dataset $\hat{f}^* \in \arg\min_{f \in \mathcal{F}} \ell(f, \Dn)$; however, several models could explain the same dataset equally well (i.e., the Rashomon effect). Rather than using a single model to compute variable importance, we propose using the entire Rashomon set.
Given a single dataset $\Dn$, we define the \textbf{Rashomon set} for a model class $\mathcal{F}$ and parameter $\varepsilon$  as the set of all models in $\mathcal{F}$ whose empirical losses are within some bound $\varepsilon > 0$ of the empirical loss minimizer:
\begin{align} \label{eqn:Rset} 
    \mathcal{R}(\epsilon, \mathcal{F}, \ell, \Dn, \lambda) = \left\{ f \in \mathcal{F} : \ell(f, \Dn; \lambda) \leq \min_{f' \in \mathcal{F}} \ell(f', \Dn; \lambda) + \varepsilon \right\}.
\end{align}
We denote this Rashomon set by $\mathcal{R}^{\varepsilon}_{\Dn}$ or ``Rset'' (this assumes a fixed $\mathcal{F}$, $\ell$ and $\lambda$). 
As discussed earlier, the Rashomon set can be fully computed and stored for non-linear models (e.g., sparse decision trees \citep{xin2022exploring}). For notational simplicity, we often omit $\lambda$ from the loss function.

While the Rashomon set describes the set of good models for a \textit{single} dataset, Rashomon sets vary across permutations (e.g., subsampling and resampling schemes) of the given data. 
We introduce a stable quantity for variable importance that accounts for all good models and all permutations from the data using the random variable $\ourvi$. $\ourvi$ is defined by its cumulative distribution function (CDF), the \textbf{Rashomon Importance Distribution ($\ourcdf$)}:  
\begin{align} \label{eq:RID}
    \ourcdf_j(k; \varepsilon, \mathcal{F}, \ell, \mathcal{P}_n, \lambda) &=\mathbb{P}_{\Dn_b \sim \mathcal{P}_n}(\ourvi_j(\varepsilon, \mathcal{F}, \ell, \Dn_b; \lambda) \leq k)\\
    &:= \mathbb{E}_{\Dn_b \sim \mathcal{P}_n}\left[\frac{|\{f \in \mathcal{R}^{\varepsilon}_{\Dn_b} : \phi_j(f, \Dn_b) \leq k\}|}{|\mathcal{R}^{\varepsilon}_{\Dn_b}|} \right]\nonumber\\
    &= \mathbb{E}_{\Dn_b \sim \mathcal{P}_n} \left[\frac{\textrm{vol of Rset s.t. variable $j$'s importance is at most $k$}}{\textrm{vol of Rset}}\right], \nonumber
\end{align}
where $\phi_j$ denotes the variable importance metric being computed on variable $j$, $k \in [\phi_{\min}, \phi_{\max}]$. For a continuous model class (e.g., linear regression models), the cardinality in the above definition becomes the volume under a measure on the function class, usually $\ell_2$ on parameter space. \ourcdf{} constructs the cumulative distribution function (CDF) for the distribution of variable importance across Rashomon sets; as $k$ increases, the value of $\mathbb{P}(\ourvi_j(\varepsilon, \mathcal{F}, \ell, \Dn_b; \lambda) \leq k)$ becomes closer to 1. The probability and expectation are taken with respect to datasets of size $n$ sampled from the empirical distribution $\mathcal{P}_n$, which is the same as considering all possible resamples of size $n$ from the originally observed dataset $\Dn$.
Equation (\ref{eq:RID}) weights the contribution of 
$\phi_j(f, \mathcal{D}^{(n)}_b)$ for each model $f$ by the proportion of datasets for which this model is a good explanation (i.e., in the Rashomon set). Intuitively, this provides greater weight to the importance of variables for stable models. 

We now define an analogous metric for the loss function $\ell$; we define the \textbf{Rashomon Loss Distribution (\RLDCDF)} evaluated at $k$ as the expected fraction of functions in the Rashomon set with loss below $k$. Here, $\RLD$ is a random variable following this CDF.
\begin{align*}
    \RLDCDF(k; \varepsilon, \mathcal{F}, \ell, \mathcal{P}_n, \lambda) &= \mathbb{P}_{\Dn_b \sim \mathcal{P}_n}\left(\RLD{}(\varepsilon, \mathcal{F}, \ell, \Dn_b; \lambda) \leq k\right)\\ 
    &:= \mathbb{E}_{\Dn_b \sim \mathcal{P}_n}\left[\frac{|\{f \in \mathcal{R}^{\varepsilon}_{\Dn_b} : \ell(f, \Dn_b) \leq k \} |}{|\mathcal{R}^{\varepsilon}_{\Dn_b}|} \right] \\
    &= \mathbb{E}_{\Dn_b \sim \mathcal{P}_n}\left[\frac{|\mathcal{R}\left(k - \min_{f \in \mathcal{F}}\ell(f, \Dn_b), \mathcal{F}, \ell; \lambda\right)|}{|\mathcal{R}(\varepsilon, \mathcal{F}, \ell; \lambda)|} \right].
\end{align*}
This quantity shows how quickly the Rashomon set ``fills up'' on average as loss changes. If there are many near-optimal functions, this will grow quickly with $k$.

In order to connect \RID{} for model class $\mathcal{F}$ to the unknown DGP $g^*$, we make a Lipschitz-continuity-style assumption on the relationship between \RLDCDF{} and $\RID$ relative to a general model class $\mathcal{F}$ and $\{g^*\}.$  To draw this connection, we define the loss CDF for $g^*$, called $LD^*$, over datasets of size $n$ as:
\begin{align*}
    LD^*(k; \ell, n, \mathcal{P}_n, \lambda) := \mathbb{E}_{\Dn_b \sim \mathcal{P}_n} \left[\mathbb{1}[\ell(g^*, \Dn_b) \leq k]\right].
\end{align*}
One could think of $LD^*$ as measuring how quickly the DGP's Rashomon set fills up as loss changes. Here, $LD^*$ is the analog of \RLDCDF{} for the data generation process.



\begin{assumption} \label{asm:lipschitz}
    If 
    \begin{align*}
        \rho\left(\RLDCDF(\cdot; \varepsilon, \mathcal{F}, \ell, \mathcal{P}_n,\lambda),
     \LD(\cdot; \ell, n, \mathcal{P}_n,\lambda) \right) &\leq \gamma \text{ then } \\
        \rho\left(\ourcdf_j(\cdot ; \varepsilon, \mathcal{F}, \ell, \mathcal{P}_n, \lambda), \ourcdf_j(\cdot ; \varepsilon, \{g^*\}, \ell, \mathcal{P}_n, \lambda) \right) &\leq d(\gamma)
    \end{align*}
    for a function $d: [0, \ell_{\max} - \ell_{\min}] \to [0, \phi_{\max} - \phi_{\min}]$ such that $\lim_{\gamma \to 0}d(\gamma)=0.$ Here, $\rho$ represents any distributional distance metric (e.g., 1-Wasserstein).
\end{assumption}

Assumption \ref{asm:lipschitz} says that a Rashomon set consisting of good approximations for $g^*$ in terms of loss will also consist of good approximations for $g^*$ in terms of variable importance.
More formally, from this assumption, we know that as $\rho(LD^*, \RLDCDF) \to 0$, the variable importance distributions will converge: $\rho\left(\ourcdf_j(\cdot, \varepsilon, \mathcal{F}, \ell, \mathcal{P}_n, \lambda), \ourcdf_j(\cdot, \varepsilon, \{g^*\}, \ell, \mathcal{P}_n, \lambda) \right) \to 0.$ We demonstrate that this assumption is realistic for a variety of model classes like linear models and generalized additive models in Section C of the supplement.



\subsection{Estimation}
We estimate $\ourcdf_j$ for each variable $j$ by leveraging bootstrap sampling to draw new datasets from the empirical data distribution: we sample observations from an observed dataset, construct its Rashomon set, and compute the $j$-th variable's importance for each model in the Rashomon set. After repeating this process for $B$ bootstrap iterations, we estimate \ourcdf{} by weighting each model $f$'s realized variable importance score (evaluated on each bootstrapped dataset) by the proportion of the bootstrapped datasets for which $f$ is in the Rashomon set \textit{and} the size of each Rashomon set in which $f$ appears.
Specifically, let $\Dn_b$ represent the dataset sampled with replacement from $\Dn$ in iteration $b = 1, \ldots, B$ of the bootstrap procedure. For each dataset $\Dn_b$, we find the Rashomon set $\mathcal{R}_{\Dn_b}^{\varepsilon}$. Finally, we compute an \textbf{empirical estimate $\widehat{\ourcdf}_j$} of $\ourcdf_j$ by computing:
\begin{align}
    \widehat{\ourcdf}_j(k; \varepsilon, \mathcal{F}, \ell, \mathcal{P}_n, \lambda)
  = \frac{1}{B}\sum_{b = 1}^B \left(\frac{|\{f \in \mathcal{R}^{\varepsilon}_{\Dn_b} : \phi_j(f, \Dn_b) \leq k \}|}{|\mathcal{R}^{\varepsilon}_{\Dn_b}|} \right). \notag
\end{align}

Under Assumption \ref{asm:lipschitz}, we can directly connect our estimate $\widehat{\ourcdf}(k; \varepsilon, \mathcal{F}, \ell, \mathcal{P}_n, \lambda)$ to the DGP's variable importance distribution $\ourcdf{}(k; \ell_{\text{max}}, \{g^*\}, \ell, \mathcal{P}_n, \lambda)$, which Theorem \ref{thm:wasser} formalizes.

\begin{theorem} \label{thm:wasser}
    Let Assumption \ref{asm:lipschitz} hold for distributional distance $\rho(A_1, A_2)$ between distributions $A_1$ and $A_2$.
    For any $t > 0$, $j \in \{0, \hdots, p\}$ as  $\rho\left(LD^*(\cdot; \ell, n, \lambda), \RLDCDF(\cdot; \varepsilon, \mathcal{F}, \ell, \mathcal{P}_n, \lambda) \right) \to 0$ and $B \to \infty$,
    \begin{align*}
        \mathbb{P}\left( \left| \widehat{\ourcdf}_j(k; \varepsilon, \mathcal{F}, \ell, \mathcal{P}_n, \lambda) - \ourcdf_j(k; \varepsilon, \{g^*\}, \ell, \mathcal{P}_n, \lambda) \right| \geq t \right) \to 0.
    \end{align*}
\end{theorem}

For a set of models that performs sufficiently well in terms of loss, $\widehat{\RID}_j$ thus recovers the CDF of variable importance for the true model across all reasonable perturbations.
Further, we can provide a finite sample bound for the estimation of a marginal distribution between $\widehat{\ourcdf}_j$ and $\ourcdf_j$ for the model class $\mathcal{F}$, as stated in Theorem \ref{thm:ourvi_finite_sample}. 
Note that this result does not require Assumption \ref{asm:lipschitz}. 

\begin{theorem} \label{thm:ourvi_finite_sample}
Let $t > 0$ and $\delta \in (0, 1)$ be some pre-specified values. Then, with probability at least $1 - \delta$ with respect to bootstrap samples of size $n$,
\begin{align}
    \left| \widehat{\ourcdf}_j(k; \varepsilon, \mathcal{F}, \ell, \mathcal{P}_n, \lambda) - \ourcdf_j(k; \varepsilon, \mathcal{F}, \ell, \mathcal{P}_n, \lambda) \right| \leq t
\end{align}
with number of bootstrap samples $B \geq \frac{1}{2t^2}\ln\left( \frac{2}{\delta} \right)$ for any $k \in [\phi_{\min}, \phi_{\max}]$.
\end{theorem}

Because we use a bootstrap procedure, we can control the number of bootstrap iterations to ensure that the difference between $\ourcdf_j$ and $\widehat{\ourcdf}_j$ is within some pre-specified error. (As defined earlier, $\ourcdf_j$ is the expectation over infinite bootstraps, whereas $\widehat{\ourcdf}_j$ is the empirical average over $B$ bootstraps.) For example, after 471 bootstrap iterations, we find that $\widehat{\ourcdf}_j(k; \varepsilon, \mathcal{F}, \ell, \mathcal{P}_n, \lambda)$ is within 0.075 of $\ourcdf_j(k; \varepsilon, \mathcal{F}, \ell, \mathcal{P}_n, \lambda)$ for any given $k$ with 90\% confidence. It also follows that as $B$ tends to infinity, the estimated $\ourvi_j$ will converge to the true value. 

Since we stably estimate the entire distribution of variable importance values, we can create (1) stable point estimates of variable importance (e.g., expected variable importance) that account for the Rashomon effect, (2) interquantile ranges of variable importance, and (3) confidence regions that characterize uncertainty around a point estimate of variable importance. We prove exponential rates of convergence for these statistics estimated using our framework in Section B of the supplement.

Because our estimand and our estimation strategy (1) enable us to manage instability, (2) account for the Rashomon effect, and (3) are completely model-agnostic and flexibly work with most existing variable importance metrics, $\ourcdf$ is a valuable quantification of variable importance. 

%% file: paper_files/synthetic_experiments.tex
\begin{figure*}[ht]
    \centering

     \includegraphics[width=\textwidth]{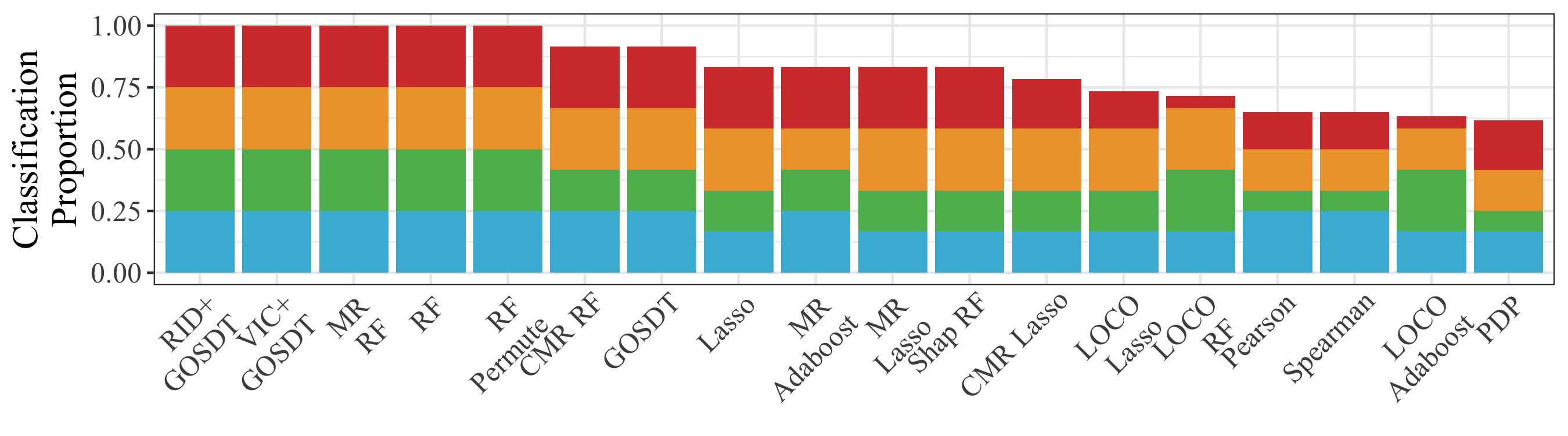}
     \begin{subfigure}{\textwidth}
         \centering
         \includegraphics[width=\textwidth]{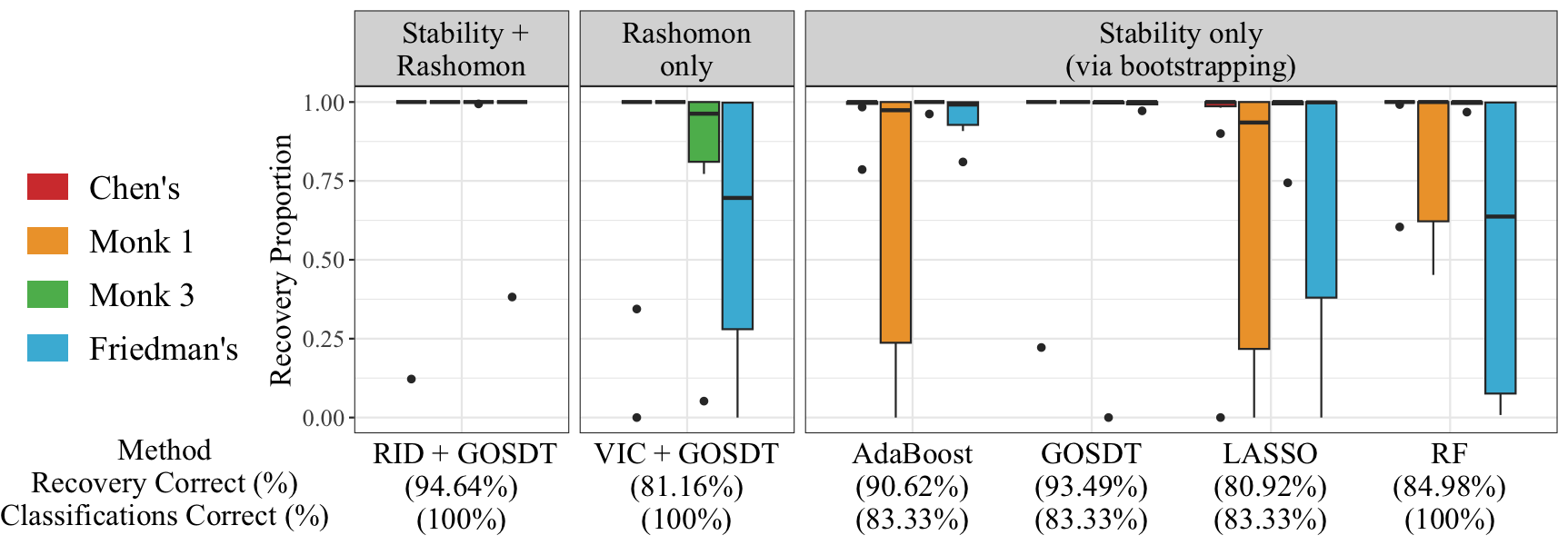}
     \end{subfigure}
    \caption{(Top) The proportion of features ranked correctly by each method on each data set represented as a \textit{stacked} barplot. 
    The figures are ordered by method performance across the four simulation setups. 
    (Bottom) The proportion of independent DGP $\phi^{(sub)}$ calculations on \textit{500 new datasets from the DGP} 
    that were 
    contained within the box-and-whiskers range computed 
    using a single training set (with bootstrapping in all methods except VIC)
    for each method and variable in each simulation. Underneath each method's label, the first row shows the percentage of times across all 500 independently generated datasets and variables that the DGP's variable importance was inside of that method's box-and-whiskers interval. The second row shows the percentage of pairwise rankings correct for each method (from the top plot). Higher is better.}
    \label{fig:baseline-rankings}
\end{figure*}

\subsection{\ourcdf\space Distinguishes Important Variables from Extraneous Variables}
\input{paper_files/recovering_rankings}

\subsection{\ourcdf\space Captures Model Reliance for the True Data Generation Process}
\input{paper_files/coverage}
\subsection{\ourcdf\space is Stable}

\input{paper_files/ablation}

%% file: paper_files/recovering_rankings.tex
There is no generally accepted ground truth measure for variable importance, so we first evaluate whether a variety of variable importance methods can correctly distinguish between variables used to generate the outcome (in a known data generation process) versus those that are not. 
We consider the following four data generation processes (DGPs). \textbf{Chen's} DGP \citep{chen2017kernel}: $ Y = \mathbb{1}[-2\sin(X_1) + \max(X_2, 0) + X_3 + \exp(-X_4) + \varepsilon \geq 2.048],$ where $X_1, \ldots, X_{10}, \varepsilon \sim \mathcal{N}(0, 1).$ Here, only $X_1, \ldots, X_4$ are relevant. \textbf{Friedman's} DGP \citep{friedman1991multivariate}: $Y = \mathbb{1}[ 10 \sin(\pi X_1 X_2) + 20 (X_3 - 0.5)^2 + 10X_4 + 5 X_5 + \varepsilon \geq 15],$ where $X_1, \ldots, X_{6} \sim \mathcal{U}(0, 1), \varepsilon \sim \mathcal{N}(0, 1).$ Here, only $X_1, \ldots, X_5$ are relevant. The \textbf{Monk 1} DGP \citep{thrun1991monk}: $Y = \max\left(\mathbb{1}[X_1 = X_2], \mathbb{1}[X_5 = 1]\right),$ where the variables $X_1, \ldots, X_6$ have domains of 2, 3, or 4 unique integer values. Only $X_1, X_2, X_5$ are important. The \textbf{Monk 3} DGP \citep{thrun1991monk}: $Y = \max\left( \mathbb{1}[X_5 = 3 \text{ and } X_4 = 1], \mathbb{1}[X_5 \neq 4 \text{ and } X_2 \neq 3] \right)$ for the same covariates in Monk 1. Also, $5\%$ label noise is added. Here, $X_2, X_4,$ and $X_5$ are relevant. 

We compare the ability of $\ourcdf$ to identify extraneous variables with that of the following baseline methods, whose details are provided in Section D of the supplement: subtractive model reliance $\phi^{\text{sub}}$ of a random forest (RF) \citep{breiman2001random}, LASSO \citep{hastie2009elements}, boosted decision trees \citep{freund1997decision}, and  generalized optimal sparse decision trees (GOSDT) \citep{lin2020generalized}; 
conditional model reliance (CMR) \citep{FisherRuDo19}; 
the impurity based model reliance metric for RF from \citep{breiman2001statistical}; the LOCO algorithm reliance \citep{lei2018distribution} for RF and Lasso; the Pearson and Spearman correlation between each feature and the outcome; the mean of the partial dependency plot (PDP) \citep{greenwell2018simple} for each feature;  the SHAP value  \citep{lundberg2018consistent} for RF; and mean of variable importance clouds (VIC) \citep{dong2020exploring} for the Rashomon set of GOSDTs \citep{xin2022exploring}. 
If we do not account for instability and simply learn a model and calculate variable importance, baseline models generally perform poorly, as shown in Section E of the supplement. 
Thus, we chose to account for instability in a way that benefits the baselines. We evaluate each baseline method for each variable across $500$ bootstrap samples and compute the \textit{median VI across bootstraps}, with the exception of VIC --- for VIC, we take the \textit{median VI value across the Rashomon set} for the original dataset, as VIC accounts for Rashomon uncertainty. Here, we aim to see whether we can identify extraneous (i.e., unimportant variables).
For a DGP with $C$ extraneous variables, we classify the $C$ variables with the $C$ smallest median variable importance values as extraneous. We repeat this experiment with different values for the Rashomon threshold $\varepsilon$ in Section E of the supplement. 

Figure \ref{fig:baseline-rankings} (top) reports the proportion of variables that are correctly classified for each simulation setup as a stacked barplot. \textbf{\ourcdf{}\space identifies all important and unimportant variables} for these complex simulations. Note that four other baseline methods -- MR RF, RF Impurity, RF Permute, and VIC -- also differentiated all important from unimportant variables. Motivated by this finding, we next explore how well
methods recover the true value for subtractive model reliance
on the DGP, allowing us to distinguish between the best performing methods on the classification task.

%% file: paper_files/coverage.tex
%
$\ourcdf$ allows us to quantify uncertainty in variable importance due to \textit{both} the Rashomon effect and instability. We perform an ablation study investigating how accounting for both stability and the Rashomon effect compares to having one without the other.
We evaluate what proportion of subtractive model reliances calculated for the DGP on 500 test sets are contained within uncertainty intervals generated using only one training dataset. This experiment tells us whether the intervals created on a single dataset will generalize.


To create the uncertainty interval on the training dataset and for each method, we first find the subtractive model reliance $\phi^{(sub)}$ across 500 bootstrap iterations of a given dataset for the four algorithms shown in Figure \ref{fig:baseline-rankings} (bottom) (baseline results without bootstrapping are in Section E of the supplementary material). Additionally, we find the VIC for the Rashomon set of GOSDTs on the original dataset. We summarize these model reliances (500 bootstraps $\times$ 28 variables across datasets $\times$ 4 algorithms + 8,247 models in VIC's + 10,840,535 total models across Rsets $\times$ 28 variables from \ourcdf) by computing their box-and-whisker ranges (1.5 $\times$ Interquartile range \cite{wiki:boxplot}). To compare with ``ground truth,'' we sample 500 test datasets from the DGP and calculate $\phi^{(sub)}$ for the DGP for that dataset. For example, assume the DGP is $Y = X^2 + \varepsilon$. We would then use $f(X) = X^2$ as our predictive model and evaluate $\phi^{(sub)}(f, \Dn)$ on $f$ for each of the 500 test sets. We then check if the box-and-whisker range of each method's interval constructed on the training set contains the computed $\phi^{(sub)}$ for the DGP for each test dataset. Doing this allows us to understand whether our interval contains the \textit{true} $\phi^{(sub)}$ for each test set.


Figure \ref{fig:baseline-rankings} (bottom) illustrates the proportion of times that the test variable importance values fell within the uncertainty intervals from training. These baselines fail to capture the test $\phi^{(sub)}$ \textit{entirely} for at least one variable ($<0.05\%$ recovery proportion).\textbf{ Only \ourcdf{} \textit{both} recovers  important/unimportant classifications perfectly and achieves a strong recovery proportion at 95\%}.

%% file: paper_files/ablation.tex
\begin{figure}[h!]
        \centering
        \includegraphics[width=0.75\textwidth]{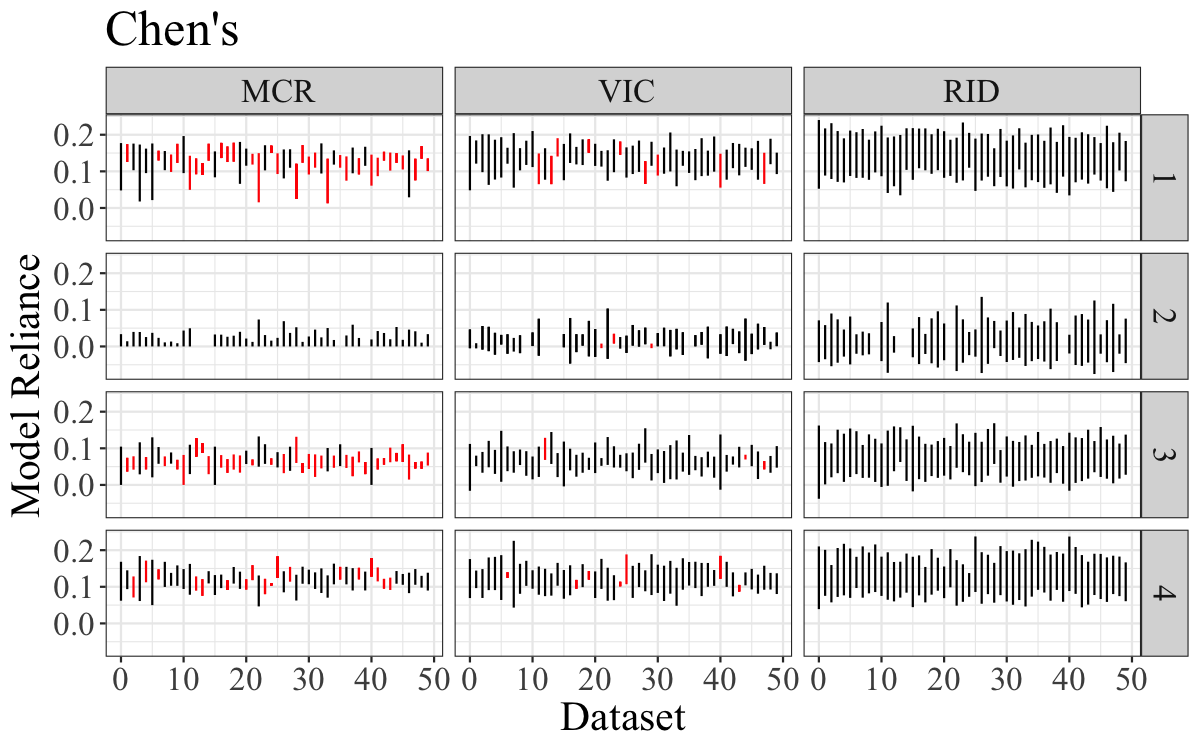}
        \caption{We generate 50 independent datasets from Chen's DGP and calculate MCR, BWRs for VIC, and BWRs for RID. 
        The above plot shows the interval for each dataset for each non-null variable in Chen's DGP. All \textcolor{red}{red}-colored intervals do not overlap with at least one of the remaining 49 intervals. }
        \label{fig:chens_overlap}
    \end{figure}
\begin{figure}[h!]
    \centering
    \includegraphics[width=0.98\textwidth]{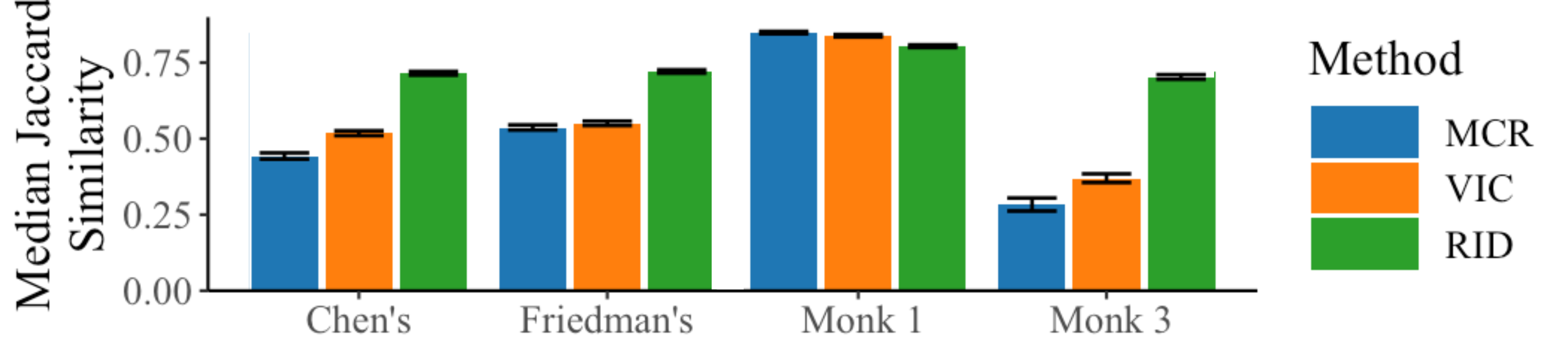}
    \caption{Median Jaccard similarity scores across 50 independently generated MCR, VIC, and \ourcdf{} box and whisker ranges for each DGP; 1 is perfect similarity. Error bars show 95\% confidence interval around the median.}
    \label{fig:iou_heatmap}
\end{figure}

Our final experiment investigates the stability of VIC and MCR (which capture only Rashomon uncertainty but not stability) to \textit{RID}, which naturally considers data perturbations. We generate 50 independent datasets from each DGP and compute the box-and-whisker ranges (BWR) of each uncertainty metric for each dataset; for every pair of BWRs for a given method, we then calculate the Jaccard similarity between BWR's. For each generated dataset, we then average the Jaccard similarity across variables. 
Figure \ref{fig:chens_overlap} shows these intervals for each non-extraneous variable from Chen's DGP. Supplement E.4 presents a similar figure for each DGP, showing that only \textit{RID}'s intervals overlap across all generations for all datasets.


Figure \ref{fig:iou_heatmap} displays the similarity scores between the box and whisker ranges of MCR, VIC, and \ourcdf{} across the 50 datasets for each DGP. Note that Monk 1 has no noise added, so instability should not be a concern for any method. For datasets including noise, \textbf{MCR and VIC achieve median similarity below 0.55; \ourcdf's median similarity is 0.69; it is much more stable}.


%% file: paper_files/case_study.tex
Having validated \textit{RID} on synthetic datasets, we demonstrate its utility in a real world application: studying which host cell transcripts and chromatin patterns are associated with high expression of Human Immunodeficiency Virus (HIV) RNA. We analyzed a dataset that combined single cell RNAseq/ATACseq profiles for 74,031 individual HIV infected cells from two different donors in the aims of finding new cellular cofactors for HIV expression that could be targeted to reactivate the latent HIV reservoir in people with HIV (PWH). A longer description of the data is in \cite{BrowneRedacted2023}. Finding results on this type of data allows us to create new hypotheses for which genes are important for HIV load prediction and might generalize to large portions of the population.
\begin{figure}[t!]
    \centering
    \includegraphics[width=.98\textwidth]{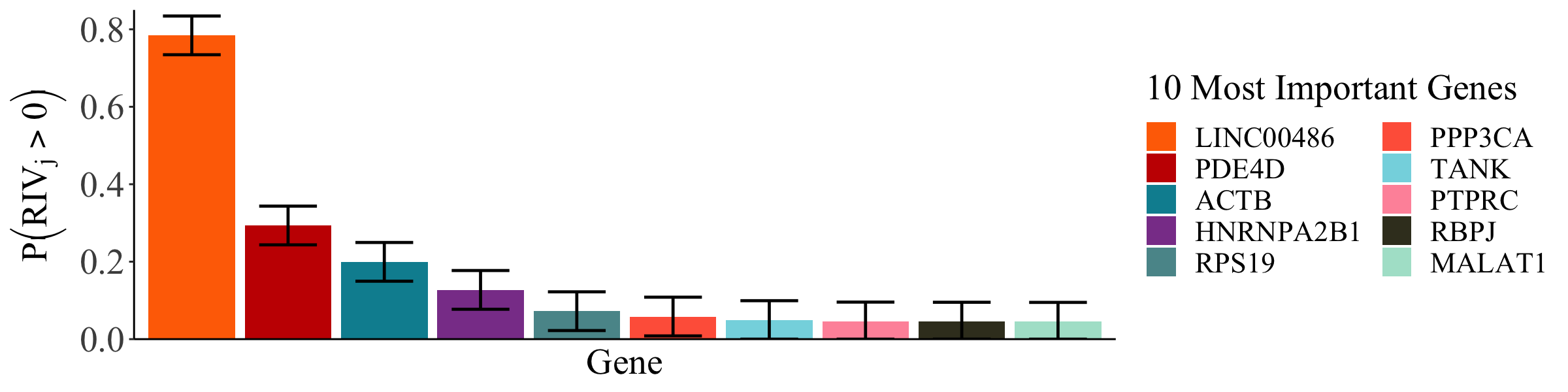}
    \caption{Probability of each gene's model reliance being greater than 0 across Rashomon sets across bootstrapped datasets for the ten genes with the highest $\mathbb{P}( \ourvi_j > 0)$. We ran 738 bootstrap iterations to ensure that $\mathbb{P}( \widehat{\ourvi}_j > 0)$ is within 0.05 of $\mathbb{P}( \ourvi_j > 0)$ with 95\% confidence (from Theorem \ref{thm:ourvi_finite_sample}).}
    \label{fig:hiv_res}
\end{figure}

To identify which genes are stably importance across good models, we evaluated this dataset using \textit{RID} over the model class of sparse decision trees using subtractive model reliance. We selected 14,614 samples (all 7,307 high HIV load samples and 7,307 random low HIV load samples) from the overall dataset in order to balance labels, and filtered the complete profiles down to the top 100 variables by individual AUC. We consider the binary classification problem of predicting high versus low HIV load. For full experimental details, see Section D of the supplement. Section E.5 of the supplement contains timing experiments for \textit{RID} using this dataset.

Figure \ref{fig:hiv_res} illustrates the probability that $\ourcdf$ is greater than 0 for the 10 highest probability variables (0 is when the variable is not important at all). \textbf{We find that LINC00486 -- a less explored gene -- is the most important variable}, with $1 - \ourcdf_{LINC00486}\left(0\right) = 78.4\%$. LINC00486 is a long non-coding RNA (i.e., it functions as an RNA molecule but does not encode a protein like most genes), and there is no literature on this gene and HIV, making this association a novel one. However, recent work \citep{wang2023nuclear} has shown that LINC00486 can enhance EBV (Epstein–Barr virus) infection by activating NF-$\kappa$B. It is well established that NF-$\kappa$B can regulate HIV expression \citep{nixon2020systemic, kretzschmar1992transcriptional, bachelerie1991hiv}, suggesting a possible mechanism and supporting future study. Notably, \ourcdf{} also highlighted PDE4D, which interacts with the Tat protein and thereby HIV transcription \citep{secchiero2000pivotal}; HNRNPA2B1, which promotes HIV expression by altering the structure of the viral promoter \citep{scalabrin2017cellular}; and MALAT1, which has recently been shown to be an important regulator of HIV expression \citep{qu2019long}. These three findings validate prior work and show that \ourcdf{} can uncover variables that are known to interact with HIV.

\textbf{Note that previous methods -- even those that account for the Rashomon effect -- could not produce this result}. MCR and VIC do not account for instability. For example, after computing MCR for 738 bootstrap iterations, we find that the MCR for the LINC00486 gene has overlap with 0 in $96.2\%$ of bootstrapped datasets, meaning MCR would not allow us to distinguish whether LINC00486 is important or not $96.2\%$ of the time. Without \textit{RID}, we would not have strong evidence that LINC00486 is necessary for good models. 
By explicitly accounting for instability, we increase trust in our analyses. 

Critically, \textit{RID} also found \textit{very low} importance for the majority of variables, allowing researchers to dramatically reduce the number of possible directions for future experiments designed to test a gene's functional role. Such experiments are time consuming and cost tens of thousands of dollars \textit{per donor}, so narrowing possible future directions to a small set of genes is of the utmost importance. \textbf{Our analysis provides a manageable set of clear directions for future work studying the functional roles of these genes in HIV.}

%% file: paper_files/conclusion.tex
We introduced \RID, a method for recovering the importance of variables in a way that accounts for both instability and the Rashomon effect. We showed that \ourcdf{} distinguishes between important and extraneous variables, and that \ourcdf{} better captures the true variable importance for the DGP than prior methods. 
We showed through a case study in HIV load prediction that \ourcdf{} can provide insight into complicated real world problems.
Our framework overcomes instability and the Rashomon effect, moving beyond variable importance for a single model and increasing reproducibility.

\RID{} can be directly computed for any model class for which the Rashomon set can be found -- at the time of publishing, decision trees, linear models, and GLMs. A limitation is that currently, there are relatively few model classes for which the Rashomon set can be computed. Therefore, future work should aim to compute and store the Rashomon set of a wider variety of model classes. Future work may investigate incorporating Rashomon sets that may be well-approximated (e.g., GAMs, \cite{chen2023understanding}), but not computed exactly, into the \ourcdf{} approach. Nonetheless, sparse trees are highly flexible, and using them with \ourcdf{} improves the trustworthiness and transparency of variable importance measures, enabling researchers to uncover important, reproducible relationships about complex processes without being misled by the Rashomon effect.

%% file: supplement_files/supplement.tex
\title{The Rashomon Importance Distribution: Getting RID of Unstable, Single Model-based Variable Importance}

\onecolumn 
\blfootnote{*Jon Donnelly and Srikar Katta contributed equally to this work.}

\maketitle



\appendix
\section*{Supplemental Material}
\section{Proofs}

\input{supplement_files/rid_emp_convergence}
\newpage
\section{Statistics Derived From \ourcdf}
\input{supplement_files/statistics_from_rid}

\newpage
\section{Example Model Classes for Which \ourcdf Converges}
\input{supplement_files/example_model_classes}

\newpage
\section{Detailed Experimental Setup}
\input{supplement_files/detailed_experimental_setup}

\newpage
\section{Additional Experiments}
\input{supplement_files/additional_experiments}

\bibliography{paper_files/references}

%% file: supplement_files/rid_emp_convergence.tex
First, recall the following assumption from the main paper:


\begin{assumption} \label{asm:lipschitz_supp}
    If 
    \begin{align*}
        \rho\left(\RLDCDF(\cdot; \varepsilon, \mathcal{F}, \ell, \mathcal{P}_n,\lambda),
     \LD(\cdot; \ell, n, \mathcal{P}_n,\lambda) \right) &\leq \gamma \text{ then } \\
        \rho\left(\ourcdf_j(\cdot ; \varepsilon, \mathcal{F}, \ell, \mathcal{P}_n, \lambda), \ourcdf_j(\cdot ; \varepsilon, \{g^*\}, \ell, \mathcal{P}_n, \lambda) \right) &\leq d(\gamma)
    \end{align*}
    for a function $d: [0, \ell_{\max} - \ell_{\min}] \to [0, \phi_{\max} - \phi_{\min}]$ such that $\lim_{\gamma \to 0}d(\gamma)=0.$ Here, $\rho$ represents any distributional distance metric (e.g., 1-Wasserstein).
\end{assumption}
    
\begin{theorem} \label{thm:wasser_supp}
    Let Assumption \ref{asm:lipschitz_supp} hold for distributional distance $\rho(A_1, A_2)$ between distributions $A_1$ and $A_2$.
    For any $t > 0$, $j \in \{0, \hdots, p\}$ as  $\rho\left(LD^*(\cdot; \ell, n, \lambda), \RLDCDF(\cdot; \varepsilon, \mathcal{F}, \ell, \mathcal{P}_n, \lambda) \right) \to 0$ and $B \to \infty$,
    \begin{align*}
        \mathbb{P}\left( \left| \widehat{\ourcdf}_j(k; \varepsilon, \mathcal{F}, \ell, \mathcal{P}_n, \lambda) - \ourcdf_j(k; \varepsilon, \{g^*\}, \ell, \mathcal{P}_n, \lambda) \right| \geq t \right) \to 0.
    \end{align*}
\end{theorem}

\begin{proof}
    Let $\Dn$ be a dataset of $n$ $(x_i, y_i)$ tuples independently and identically distributed according to the empirical distribution $\mathcal{P}_n.$ Let $k \in [\phi_{\min}, \phi_{\max}].$ 

    Then, we know that
    \begin{align*}
        &\mathbb{P}\left( \left| \widehat{\ourcdf}_j(k; \varepsilon, \mathcal{F}, \ell, \mathcal{P}_n, \lambda) - \ourcdf_j(k; \varepsilon, \{g^*\}, \ell, \mathcal{P}_n, \lambda) \right| \geq t \right) \\
       =&\mathbb{P}\Big( \Big| \widehat{\ourcdf}_j(k; \varepsilon, \mathcal{F}, \ell, \mathcal{P}_n, \lambda) - \ourcdf_j(k; \varepsilon, \mathcal{F}, \ell, \mathcal{P}_n, \lambda) \\
       &+ \ourcdf_j(k; \varepsilon, \mathcal{F}, \ell, \mathcal{P}_n, \lambda)  - \ourcdf_j(k; \varepsilon, \{g^*\}, \ell, \mathcal{P}_n, \lambda) \Big| \geq t \Big)  \\
       &\text{ (by adding 0)} \\
       \leq&\mathbb{P}\Big( \left| \widehat{\ourcdf}_j(k; \varepsilon, \mathcal{F}, \ell, \mathcal{P}_n, \lambda) - \ourcdf_j(k; \varepsilon, \mathcal{F}, \ell, \mathcal{P}_n, \lambda)  \right| \\ 
       &+ \left|\ourcdf_j(k; \varepsilon, \mathcal{F}, \ell, \mathcal{P}_n, \lambda) - \ourcdf_j(k; \varepsilon, \{g^*\}, \ell, \mathcal{P}_n, \lambda) \right| \geq t \Big) \\
       &\text{ (by the triangle inequality) } \\
       \leq&\mathbb{P}\left( \left| \widehat{\ourcdf}_j(k; \varepsilon, \mathcal{F}, \ell, \mathcal{P}_n, \lambda) - \ourcdf_j(k; \varepsilon, \mathcal{F}, \ell, \mathcal{P}_n, \lambda) \right| \geq \frac{t}{2} \right) \\ 
       &+ \mathbb{P}\left(\left |  \ourcdf_j(k; \varepsilon, \mathcal{F}, \ell, \mathcal{P}_n, \lambda) - \ourcdf_j(k; \varepsilon, \{g^*\}, \ell, \mathcal{P}_n, \lambda) \right| \geq \frac{t}{2} \right) \\
       &\text{(by union bound)}.
    \end{align*}
     Recall that, in the theorem statement, we have assumed $\rho\left(LD^*(\cdot; \ell, n, \lambda), \RLDCDF(\cdot; \varepsilon, \mathcal{F}, \ell, \mathcal{P}_n, \lambda) \right) \to 0.$ Therefore, by Assumption \ref{asm:lipschitz_supp},
    \begin{align*}
        \mathbb{P}\left(\left|\ourcdf_j(k; \varepsilon, \mathcal{F}, \ell, \mathcal{P}_n, \lambda) - \ourcdf_j(k; \varepsilon, \{g^*\}, \ell, \mathcal{P}_n, \lambda) \right| \geq \frac{t}{2} \right) \to 0.
    \end{align*}

    Additionally, we will show in Corollary \ref{cor:rid_consistency} that as $B \to \infty,$ 
    \begin{align*}
        \mathbb{P}\left(\left | \widehat{\ourcdf}_j(k; \varepsilon, \mathcal{F}, \ell, \mathcal{P}_n, \lambda) - \ourcdf_j(k; \varepsilon, \mathcal{F}, \ell, \mathcal{P}_n, \lambda)\right| \geq \frac{t}{2} \right) \to 0.
    \end{align*}
    Therefore, as $B \to \infty$ and $\rho\left(LD^*(\cdot; \ell, n, \lambda), \RLDCDF(\cdot; \varepsilon, \mathcal{F}, \ell, \mathcal{P}_n, \lambda) \right) \to 0,$ the \textit{estimated} Rashomon importance distribution for model class $\mathcal{F}$ converges to the true Rashomon importance distribution for the DGP $g^*.$
\end{proof}

\begin{theorem} \label{thm:rid_finite_sample}
Let $\Dn$ be a dataset of $n$ $(x_i, y_i)$ tuples independently and identically distributed according to the empirical distribution $\mathcal{P}_n.$ Let $k \in [\phi_{\min}, \phi_{\max}].$ Then, with probability $1 - \delta$, with $B \geq \frac{1}{2t^2}\ln\left( \frac{2}{\delta} \right)$ bootstrap replications,
\begin{align*}
     \left| \widehat{\ourcdf}_j(k) - \ourcdf_j(k) \right| < t.
\end{align*}
\end{theorem}

\begin{proof}
First, let us restate the definition of $\ourcdf_j$ and $\widehat{\ourcdf}_j.$ Let $n \in \mathbb{N}.$ Let $\varepsilon$ be the Rashomon threshold, and let the Rashomon set for some dataset $\Dn$ and some fixed model class $\mathcal{F}$ be denoted as $\mathcal{R}_{\Dn}^\varepsilon$. Without loss of generality, assume $\mathcal{F}$ is a finite model class. Then, for a given $k \in [\phi_{\min}, \phi_{\max}]$,
\begin{align*}
      \ourcdf_j(k; \varepsilon, \mathcal{F}, \ell, \mathcal{P}_n,\lambda) 
    &= \mathbb{E}_{\Dn \sim \mathcal{P}_n}\left[ \frac{\sum_{f \in \mathcal{R}_{\Dn}^\varepsilon} \mathbb{1}[\phi_j(f, \Dn) \leq k]}{| \mathcal{R}_{\Dn}^\varepsilon|} \right].
\end{align*}
Note that the expectation is over all datasets of size $n$ sampled with replacement from the originally observed dataset, represented by $\mathcal{P}_n;$ we are taking the expectation over bootstrap samples.

We then sample datasets of size $n$ with replacement from the \textit{empirical} CDF $\mathcal{P}_n$, find the Rashomon set for the replicate dataset, and compute the variable importance metric for each model in the discovered Rashomon set. For the same $k \in [\phi_{\min}, \phi_{\max}],$
\begin{align*}
    \widehat{\ourcdf}_j(k; \varepsilon, \mathcal{F}, \ell, \mathcal{P}_n, \lambda)
    &= \frac{1}{B}\sum_{\Dn_b \sim \mathcal{P}_n}\left[ \frac{\sum_{f \in \mathcal{R}_{\Dn_b}^\varepsilon} \mathbb{1}[\phi_j(f, \Dn_b) \leq k]}{| \mathcal{R}_{\Dn_b}^\varepsilon|} \right],
\end{align*}
where $B$ represents the number of size $n$ datasets sampled from $\mathcal{P}_n$.

Notice that
\begin{align} \label{eqn:inside_rid}
    0 \leq \frac{\sum_{f \in \mathcal{R}_{\Dn}^\varepsilon} \mathbb{1}[\phi_j(f, \Dn) \leq k]}{| \mathcal{R}_{\Dn}^\varepsilon|} \leq 1.
\end{align}
Because $\widehat{\ourcdf}_j(k; \varepsilon, \mathcal{F}, \ell, \mathcal{P}_n, \lambda)$ is an Euclidean average of the quantity in Equation \eqref{eqn:inside_rid} and $\ourcdf_j(k;\varepsilon, \mathcal{F}, \ell, \mathcal{P}_n, \lambda)$ is the expectation of the quantity in Equation \eqref{eqn:inside_rid}, we can use Hoeffding's inequality to show that
\begin{align*}
    &\mathbb{P}\left( \left| \widehat{\ourcdf}_j(k; \varepsilon, \mathcal{F}, \ell, \mathcal{P}_n, \lambda) -  \ourcdf_j(k;\varepsilon, \mathcal{F}, \ell, \mathcal{P}_n, \lambda)\right| > t \right) \\
    \leq &2\exp\left( -2Bt^2\right)
\end{align*}
for some $t > 0.$ 

Now, we can manipulate Hoeffding's inequality to discover a finite sample bound. Instead of setting $B$ and $t,$ we will now find the $B$ necessary to guarantee that
\begin{align}
    \mathbb{P}\left( \left| \widehat{\ourcdf}_j(k; \varepsilon, \mathcal{F}, \ell, \mathcal{P}_n, \lambda) - \ourcdf_j(k;\varepsilon, \mathcal{F}, \ell, \mathcal{P}_n, \lambda)  \right| \geq  t\right) \leq \delta
\end{align}
for some $\delta, t > 0.$ 

Let $\delta > 0$. From Hoeffding's inequality, we see that if we choose $B$ such that $2\exp\left(-2Bt^2 \right) \leq \delta,$ then 
\begin{align*} 
    \mathbb{P}\left( \left| \widehat{\ourcdf}_j(k; \varepsilon, \mathcal{F}, \ell, \mathcal{P}_n, \lambda) - \ourcdf_j(k;\varepsilon, \mathcal{F}, \ell, \mathcal{P}_n, \lambda)  \right| \geq  t\right) \leq 2\exp\left(-2Bt^2\right) \leq \delta.
\end{align*}
Notice that $2\exp\left(-2Bt^2 \right) \leq \delta$ if and only if $B \geq \frac{1}{2t^2}\ln\left( \frac{2}{\delta} \right).$

Therefore, with probability $1 - \delta,$
\begin{align*}
    \left|\widehat{\ourcdf}_j(k) - \ourcdf_j(k) \right| \leq t
\end{align*}
with $B \geq \frac{1}{2t^2}\ln\left( \frac{2}{\delta} \right)$ bootstrap iterations.
\end{proof}

\begin{corollary} \label{cor:rid_consistency}
Let $t > 0$,  $k \in [\phi_{\min}, \phi_{\max}],$ and assume that $\Dn \sim\mathcal{P}_n.$ As $B \to \infty,$ 
\begin{align*}
    \mathbb{P}\left( \left| \widehat{\ourcdf}_j(k; \varepsilon, \mathcal{F}, \ell, \mathcal{P}_n, \lambda) - \ourcdf_j(k;\varepsilon, \mathcal{F}, \ell, \mathcal{P}_n, \lambda)  \right| \geq  t\right) \to 0.
\end{align*}
\end{corollary}

\begin{proof}
    Recall the results of Theorem \ref{thm:rid_finite_sample}:

    \begin{align*}
        &\mathbb{P}\left( \left| \widehat{\ourcdf}_j(k; \varepsilon, \mathcal{F}, \ell, \mathcal{P}_n, \lambda) - \ourcdf_j(k;\varepsilon, \mathcal{F}, \ell, \mathcal{P}_n, \lambda)  \right| \geq  t\right) \\
        &\leq 2\exp\left(-2Bt^2 \right) \\
        &\to 0 \text{ as B } \to \infty.
    \end{align*}
    
\end{proof}

%% file: supplement_files/statistics_from_rid.tex
\begin{corollary} \label{cor:ourvi_expectation_finite_sample}
    Let $\varepsilon, B > 0.$ Then,
    \begin{align}
        \mathbb{P}\left(\left| \mathbb{E}[\ourcdf_j] - \mathbb{E}[\widehat{\ourcdf}_j] \right| \geq \varepsilon_E \right) \leq 2\exp\left(\frac{-2B\varepsilon_E^2}{(\phi_{\max} - \phi_{\min})^2}\right).
    \end{align}
    Therefore, the expectation of $\widehat{\ourvi}_j$ converges exponentially quickly to the expectation of $\ourvi_j.$ The notation $\mathbb{E}[\ourcdf_j]$ denotes the expectation of the random variable distributed according to $\ourcdf_j.$
\end{corollary}

\begin{proof}
    Let $\phi_{\min}, \phi_{\max}$ represent the bounds of the variable importance metric $\phi$. Assume that $0 \leq \phi_{\min} \leq \phi_{\max} < \infty$. If $\phi_{\min} < 0$, then we can modify the variable importance metric to be strictly positive; for example, if $\phi$ is Pearson correlation -- which has a range between -1 and 1 -- we can define a new variable importance metric that is the absolute value of the Pearson correlation \textit{or} define another metric that is the Pearson correlation plus 1 so that the range is now bounded below by 0.

    Now, recall that for any random variable $X$ whose support is strictly greater than 0, we can calculate its expectation as $\mathbb{E}_X[X] = \int_0^\infty(1 - \mathbb{P}(X \leq x)) dx.$ Because $\phi_{\min} \geq 0,$ we know that
    \begin{align*}
        &\mathbb{E}[\ourcdf_j] \\
        =& \int_{\phi_{\min}}^{\phi_{\max}} \left(1 - \mathbb{P}(\ourvi_j \leq k) \right) dk \\
        =& \int_{\phi_{\min}}^{\phi_{\max}} \left(1 - \mathbb{E}_{D^{(n)}}\left[ \sum_{f \in \mathcal{F}} \frac{\mathbb{1}[f \in \mathcal{R}_{\Dn}^{\varepsilon}]\mathbb{1}[\phi_j(f, D^{(n)}) \leq k]}{\sum_{f \in \mathcal{F}}\mathbb{1}[f \in \mathcal{R}_{\Dn}^{\varepsilon}]} \right] \right) dk \\
        =& \int_{\phi_{\min}}^{\phi_{\max}} dk - \int_{\phi_{\min}}^{\phi_{\max}} \mathbb{E}_{D^{(n)}}\left[ \sum_{f \in \mathcal{F}} \frac{\mathbb{1}[f \in \mathcal{R}_{\Dn}^{\varepsilon}]\mathbb{1}[\phi_j(f, D^{(n)}) \leq k]}{\sum_{f \in \mathcal{F}}\mathbb{1}[f \in \mathcal{R}_{\Dn}^{\varepsilon}]} \right] dk \\
        =& \left(\phi_{\max} - \phi_{\min}\right)  - \mathbb{E}_{D^{(n)}}\left[\int_{\phi_{\min}}^{\phi_{\max}}  \sum_{f \in \mathcal{F}} \frac{\mathbb{1}[f \in \mathcal{R}_{\Dn}^{\varepsilon}]\mathbb{1}[\phi_j(f, D^{(n)}) \leq k]}{\sum_{f \in \mathcal{F}}\mathbb{1}[f \in \mathcal{R}_{\Dn}^{\varepsilon}]} dk\right] \text{ by Fubini's theorem}.
    \end{align*}
    Using similar logic we can show that
    \begin{align*}
        \mathbb{E}[\widehat{\ourcdf}_j] 
        &= \int_{\phi_{\min}}^{\phi_{\max}} \left(1 - \frac{1}{B} \sum_{b = 1}^B \sum_{f \in \mathcal{F}} \frac{\mathbb{1}[f \in \mathcal{R}_{\Dn}^{\varepsilon}]\mathbb{1}[\phi_j(f, D_b^{(n)}) \leq k]}{\sum_{f \in \mathcal{F}}\mathbb{1}[f \in \mathcal{R}_{\Dn}^{\varepsilon}]} \right) dk \\
        &= \int_{\phi_{\min}}^{\phi_{\max}} dk - \int_{\phi_{\min}}^{\phi_{\max}} \frac{1}{B} \sum_{b = 1}^B \sum_{f \in \mathcal{F}} \frac{\mathbb{1}[f \in \mathcal{R}_{\Dn}^{\varepsilon}]\mathbb{1}[\phi_j(f, D_b^{(n)}) \leq k]}{\sum_{f \in \mathcal{F}}\mathbb{1}[f \in \mathcal{R}_{\Dn}^{\varepsilon}]} dk \\
        &= \left(\phi_{\max} - \phi_{\min}\right) - \frac{1}{B} \sum_{b = 1}^B \left(\int_{\phi_{\min}}^{\phi_{\max}} \sum_{f \in \mathcal{F}} \frac{\mathbb{1}[f \in \mathcal{R}_{\Dn}^{\varepsilon}]\mathbb{1}[\phi_j(f, D_b^{(n)}, m) \leq k]}{\sum_{f \in \mathcal{F}}\mathbb{1}[f \in \mathcal{R}_{\Dn}^{\varepsilon}]} dk\right).
    \end{align*}

    We can then rewrite $\left|\mathbb{E}[\ourcdf_j]  - \mathbb{E}[\widehat{\ourcdf}_j] \right|$ using the calculations above: 
    \begin{align*}
        &\left|\mathbb{E}[\ourcdf_j]  - \mathbb{E}[\widehat{\ourcdf_j}] \right| \\
    =& \Bigg| \left(\phi_{\max} - \phi_{\min}\right)  - \mathbb{E}_{D^{(n)} \sim \mathcal{P}_n}\left[\int_{\phi_{\min}}^{\phi_{\max}}  \sum_{f \in \mathcal{F}} \frac{\mathbb{1}[f \in \mathcal{R}_{\Dn}^{\varepsilon}]\mathbb{1}[\phi)j(f, D_b^{(n)}) \leq k]}{\sum_{f \in \mathcal{F}}\mathbb{1}[f \in \mathcal{R}_{\Dn}^{\varepsilon}]} dk\right] \\
    -& \left(\left(\phi_{\max} - \phi_{\min}\right)  - \frac{1}{B} \sum_{b = 1}^B \left(\int_{\phi_{\min}}^{\phi_{\max}} \sum_{f \in \mathcal{F}} \frac{\mathbb{1}[f \in \mathcal{R}_{\Dn}^{\varepsilon}]\mathbb{1}[\phi_j(f, D_b^{(n)}) \leq k]}{\sum_{f \in \mathcal{F}}\mathbb{1}[f \in \mathcal{R}_{\Dn}^{\varepsilon}]} dk\right) \right) \Bigg| \\
    &= \Bigg| -\mathbb{E}_{D^{(n)} \sim \mathcal{P}_n}\left[\int_{\phi_{\min}}^{\phi_{\max}}  \sum_{f \in \mathcal{F}} \frac{\mathbb{1}[f \in \mathcal{R}_{\Dn}^{\varepsilon}]\mathbb{1}[\phi_j(f, D_b^{(n)}) \leq k]}{\sum_{f \in \mathcal{F}}\mathbb{1}[f \in \mathcal{R}_{\Dn}^{\varepsilon}]} dk\right] \\
    &+ \frac{1}{B} \sum_{b = 1}^B \left(\int_{\phi_{\min}}^{\phi_{\max}} \sum_{f \in \mathcal{F}} \frac{\mathbb{1}[f \in \mathcal{R}_{\Dn}^{\varepsilon}]\mathbb{1}[\phi_j(f, D_b^{(n)}) \leq k]}{\sum_{f \in \mathcal{F}}\mathbb{1}[f \in \mathcal{R}_{\Dn}^{\varepsilon}]} dk\right) \Bigg|.
    \end{align*}
    Because $0 \leq \mathbb{P}(\ourvi_j \leq k), \mathbb{P}(\widehat{\ourvi_j} \leq k) \leq 1$ for all $k \in \mathbb{R}$, 
    \begin{align*}
        \int_{\phi_{\min}}^{\phi_{\max}} 0 dk &\leq \int_{\phi_{\min}}^{\phi_{\max}} \mathbb{P}(\ourvi_j \leq k) dk, \int_{\phi_{\min}}^{\phi_{\max}} \mathbb{P}(\widehat{\ourvi}_j \leq k) dk \leq \int_{\phi_{\min}}^{\phi_{\max}} 1 dk \\
        0 &\leq \int_{\phi_{\min}}^{\phi_{\max}} \mathbb{P}(\ourvi_j \leq k) dk, \int_{\phi_{\min}}^{\phi_{\max}} \mathbb{P}(\widehat{\ourvi_j} \leq k) dk \leq (\phi_{\max} - \phi_{\min}),
    \end{align*}
    suggesting that $\left(\int_{\phi_{\min}}^{\phi_{\max}}  \frac{\sum_{f \in \mathcal{F}}\mathbb{1}[f \in \mathcal{R}_{\Dn}^{\varepsilon}]\mathbb{1}[\phi_j(f, D_b^{(n)}) \leq k]}{\sum_{f \in \mathcal{F}}\mathbb{1}[f \in \mathcal{R}_{\Dn}^{\varepsilon}]} dk\right)$ is bounded.

    Then, by Hoeffding's inequality, we know that
    \begin{align*}
        &\mathbb{P}\left(\left|\mathbb{E}[\ourvi_j]  - \mathbb{E}[\widehat{\ourvi}_j] \right| > \varepsilon_E \right) \\
       =&  \mathbb{P}\Bigg(\Bigg| \mathbb{E}_{D^{(n)} \sim \mathcal{P}_n}\left[\int_{\phi_{\min}}^{\phi_{\max}}  \sum_{f \in \mathcal{F}} \frac{\mathbb{1}[f \in \mathcal{R}_{\Dn}^{\varepsilon}]\mathbb{1}[\phi_j(f, D_b^{(n)}) \leq k]}{\sum_{f \in \mathcal{F}}\mathbb{1}[f \in \mathcal{R}_{\Dn}^{\varepsilon}]} dk\right] \\
    &- \frac{1}{B} \sum_{b = 1}^B \left(\int_{\phi_{\min}}^{\phi_{\max}} \sum_{f \in \mathcal{F}} \frac{\mathbb{1}[f \in \mathcal{R}_{\Dn}^{\varepsilon}]\mathbb{1}[\phi_j(f, D_b^{(n)}) \leq k]}{\sum_{f \in \mathcal{F}}\mathbb{1}[f \in \mathcal{R}_{\Dn}^{\varepsilon}]} dk\right) \Bigg| > \varepsilon_E \Bigg) \\
    \leq& 2\exp\left( \frac{-2B\varepsilon_E^2}{(\phi_{\max} - \phi_{\min})^2} \right).
    \end{align*}
\end{proof}

\begin{corollary} \label{cor:iqr}
        
        Assume $\widehat{\ourcdf}_j(k)$ and $\ourcdf_j(k)$ are \textit{strictly} increasing in $k \in [\phi_{\min}, \phi_{\max}].$ Then, the interquantile range (IQR) of $\widehat{\ourcdf}_j$ will converge in probability to the IQR of $\ourcdf_j$.
    \end{corollary}
    \begin{proof}
        Let $k_{0.25}$ be the $k$ such that $\ourcdf_j(k_{0.25}) = 0.25.$ And let $k_{0.75}$ be the $k$ such that $\ourcdf_j(k_{0.75}) = 0.75.$ Similarly, let $\hat{k}_{0.25}$ be the $k$ such that $\widehat{\ourcdf}_j(\hat{k}_{0.25}) = 0.25.$ And let $\hat{k}_{0.75}$ be the $k$ such that $\widehat{\ourcdf}_j(\hat{k}_{0.75}) = 0.75.$ The IQR of $\widehat{\ourcdf}_j$ converges to the IQR of $\ourcdf_j$ if $\hat{k}_{0.25} \to k_{0.25}$ and $\hat{k}_{0.75} \to k_{0.75}.$

        Because $\widehat{\ourcdf}_j(k)$ and $\ourcdf_j(k)$ are increasing in $k,$ we know that if $
            \mathbb{P}\left(\widehat{\ourvi}_j \leq k_{0.25}  \right) = 0.25,$
        then $\hat{k}_{0.25} = k_{0.25}.$ An analogous statement holds for $\hat{k}_{0.75}.$

        So, we will bound how far $\widehat{\ourcdf}_j(k_{0.25})$ is from $0.25 = \ourcdf_j(k_{0.25})$ and how far $\widehat{\ourcdf}_j(k_{0.75})$ is from $0.75 = \ourcdf_j(k_{0.75}).$

        Let $t > 0.$ Then,
        \begin{align*}
            &\mathbb{P}\left( \left|\widehat{\ourcdf}_j(k_{0.25}) - \ourcdf_j(k_{0.25})  \right| + \left| \widehat{\ourcdf}_j(k_{0.75}) - \ourcdf_j(k_{0.75}) \right| > t \right) \\
            \leq &\mathbb{P}\left( \left\{ \left|\widehat{\ourcdf}_j(k_{0.25}) - \ourcdf_j(k_{0.25})  \right| > \frac{t}{2}\right\}  \cup \left\{ \left| \widehat{\ourcdf}_j(k_{0.75}) - \ourcdf_j(k_{0.75}) \right| > \frac{t}{2}\right\}  \right) \\
            \leq &\mathbb{P}\left( \left\{ \left|\widehat{\ourcdf}_j(k_{0.25}) - \ourcdf_j(k_{0.25})  \right| > \frac{t}{2}\right\} \right) \\ 
            &+ \mathbb{P}\left(\left\{ \left| \widehat{\ourcdf}_j(k_{0.75}) - \ourcdf_j(k_{0.75}) \right| > \frac{t}{2} \right\} \right) \text{ by Union bound}.
        \end{align*}
        Then, by Theorem \ref{thm:rid_finite_sample},
        \begin{align*}
            \mathbb{P}\left( \left|\widehat{\ourcdf}_j(k_{0.25}) - \ourcdf_j(k_{0.25})  \right|  > \frac{t}{2}\right) \leq 2\exp\left(-2B\frac{t^2}{4} \right).
        \end{align*}
        So,
        \begin{align*}
            &\mathbb{P}\left( \left|\widehat{\ourcdf}_j(k_{0.25}) - \ourcdf_j(k_{0.25})  \right| + \left| \widehat{\ourcdf}_j(k_{0.75}) - \ourcdf_j(k_{0.75}) \right| > t \right) \\
            \leq &\mathbb{P}\left( \left\{ \left|\widehat{\ourcdf}_j(k_{0.25}) - \ourcdf_j(k_{0.25})  \right| \right\} > \frac{t}{2}\right) \\ 
            &+ \mathbb{P}\left(\left\{ \left| \widehat{\ourcdf}_j(k_{0.75}) - \ourcdf_j(k_{0.75}) \right| \right\} > \frac{t}{2} \right) \\
            \leq & 2\exp\left(-2B\frac{t^2}{4} \right) + 2\exp\left(-2B\frac{t^2}{4} \right) \\
            = &4\exp\left(-2B\frac{t^2}{4} \right).
        \end{align*}
    So, as $B \to \infty, \mathbb{P}\left( \left|\widehat{\ourcdf}_j(k_{0.25}) - \ourcdf_j(k_{0.25})  \right| + \left| \widehat{\ourcdf}_j(k_{0.75}) - \ourcdf_j(k_{0.75}) \right| > t \right) $ ultimately converging to 0.

    Therefore, the IQR of $\widehat{\ourcdf}_j$ converges to the IQR of $\ourcdf$.

    \end{proof}

%% file: supplement_files/example_model_classes.tex
First, recall the following assumption from the main paper:
\setcounter{assumption}{0}

\begin{assumption} 
    If 
    \begin{align*}
        \rho\left(\RLDCDF(\cdot; \varepsilon, \mathcal{F}, \ell, \mathcal{P}_n,\lambda),
     \LD(\cdot; \ell, n, \mathcal{P}_n,\lambda) \right) &\leq \gamma \text{ then } \\
        \rho\left(\ourcdf_j(\cdot ; \varepsilon, \mathcal{F}, \ell, \mathcal{P}_n, \lambda), \ourcdf_j(\cdot ; \varepsilon, \{g^*\}, \ell, \mathcal{P}_n, \lambda) \right) &\leq d(\gamma)
    \end{align*}
    for a function $d: [0, \ell_{\max} - \ell_{\min}] \to [0, \phi_{\max} - \phi_{\min}]$ such that $\lim_{\gamma \to 0}d(\gamma)=0.$ Here, $\rho$ represents any distributional distance metric (e.g., 1-Wasserstein).
\end{assumption}
In this section, we highlight two simple examples of model classes and model reliance metrics for which Assumption \ref{asm:lipschitz_supp} holds. First we show that Assumption \ref{asm:lipschitz_supp} holds for the class of linear regression models with the model reliance metric being the coefficient assigned to each variable in Proposition \ref{prop:linear}; Proposition \ref{prop:gam} presents a similar result for generalized additive models. 
We begin by presenting two lemmas which will help prove Proposition \ref{prop:linear}:

\begin{lemma}
\label{lem:rld_convex}
    Let $\ell$ be unregularized mean square error, used as the objective for estimating optimal models in some class of continuous models $\mathcal{F}$.  Assume that the DGP's noise $\epsilon$ is centered at 0: $\mathbb{E}[\epsilon] = 0.$ Define the function $m : [0, \ell_{\text{\rm max}}] \to [0, 1]$ as: 
   \begin{align*}
        m(\varepsilon) := \lim_{n \to \infty}\int_{\ell_{\min}}^{\ell_{\max}} \left|\LD(k; \ell, n, \mathcal{P}_n,\lambda) - \RLDCDF(k; \varepsilon, \mathcal{F}, \ell, \mathcal{P}_n, \lambda) \right| dk.\\
    \end{align*}
    The function $m$ is a strictly increasing function of $\varepsilon$; $m$ simply measures the integrated absolute error between the CDF of $g^*$'s loss distribution and the CDF of the Rashomon set's loss distribution.
    Then, if $g^* \in \mathcal{F}$, then $m(0) = 0$.
\end{lemma}

\begin{proof}
    Let $\ell$ be unregularized mean square error, used as the objective for estimating optimal models in some class of continuous models $\mathcal{F}$. Let $g^*$ denote the unknown DGP. Throughout this proof, we consider the setting with $n \to \infty$, although we often omit this notation for simplicity. 

    First, we restate the definition of $\RLDCDF$ and $\LD$ for reference:
    \begin{align*}
        \RLDCDF(k; \varepsilon, \mathcal{F}, \ell, \mathcal{P}_n, \lambda) &:= \mathbb{E}_{\Dn \sim \mathcal{P}_n}\left[\frac{\nu(\{f \in \mathcal{R}^{\varepsilon}_{\Dn} : \ell(f, \Dn) \leq k \} )}{\nu(\mathcal{R}^{\varepsilon}_{\Dn})} \right] 
    \end{align*}
    and
    
\begin{align*}
    \LD(k; \ell, n, \mathcal{P}_n,\lambda) := \mathbb{E}_{\Dn \sim \mathcal{P}_n} \left[\mathbb{1}[\ell(g^*, \Dn) \leq k]\right].
\end{align*}

    Because $g^*$ is the DGP, we know that its expected loss should be lower than the expected loss for any other model in the model class: $\mathbb{E}_{\Dn \sim \mathcal{P}_n}[\ell(g^*, \Dn)] \leq \mathbb{E}_{\Dn \sim \mathcal{P}_n}[\ell(f, \Dn)]$ for any $f \in \mathcal{F}$ such that $f \neq g^*$, as we have assumed that any noise has expectation 0. For simplicity, we denote $\mathbb{E}_{\Dn \sim \mathcal{P}_n}[\ell(g^*, \Dn)]$ by $\ell^*$. We first show that $m$ is monotonically increasing in $\varepsilon$ by showing that, for any $\varepsilon > \varepsilon' \geq 0$: 
    \begin{align*}
        &\lim_{n \to \infty}\int_{\ell_{\min}}^{\ell_{\max}} \left|\LD(k; \ell, n, \mathcal{P}_n,\lambda) - \RLDCDF(k; \varepsilon, \mathcal{F}, \ell, \mathcal{P}_n, \lambda) \right| dk\\
        >
        &\lim_{n \to \infty} \int_{\ell_{\min}}^{\ell_{\max}} \left|\LD(k; \ell, n, \mathcal{P}_n,\lambda) - \RLDCDF(k; \varepsilon', \mathcal{F}, \ell, \mathcal{P}_n, \lambda) \right| dk
    \end{align*} 
    by demonstrating that the inequality holds for each individual value of $k$. First, note that:
    \begin{align*}
        \LD(k; \ell, n, \mathcal{P}_n,\lambda) &= \mathbb{E}_{\Dn \sim \mathcal{P}_n}\left[\mathbb{1}[\ell(g^*, \Dn) \leq k] \right].
    \end{align*}
    As $n \to \infty$, this quantity approaches
    \begin{align*}
        \mathbb{E}_{\Dn \sim \mathcal{P}_n}\left[\mathbb{1}[\ell(g^*, \Dn) \leq k] \right] &= \mathbb{1}[\ell(g^*, \Dn) \leq k].
    \end{align*}

    \begin{figure}
        \centering
        \includegraphics[width=0.95\textwidth]{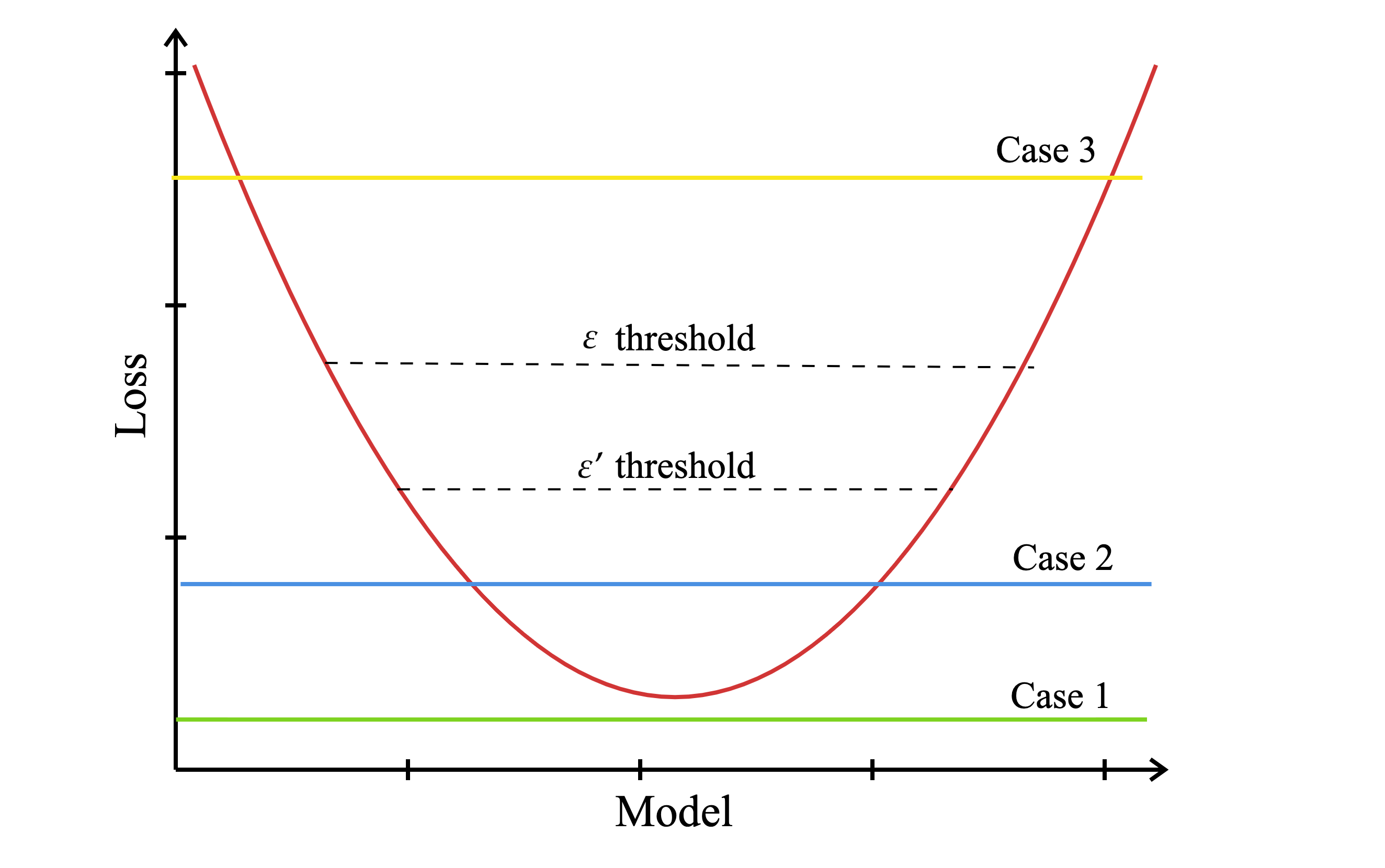}
        \caption{A visual overview of the proof of Lemma \ref{lem:rld_convex}. In \textbf{Case 1}, we consider loss values that are achieved by no models in the model class, so each loss distribution has 0 mass below $k$ in this case. \textbf{Case 2} covers each value of $k$ such that $k$ is larger than $\ell^*$, so $\LD(k) = 1$. The $\RLDCDF$ for the $\varepsilon'$ Rashomon set is closer to 1 than the $\varepsilon$ Rashomon set because a larger proportion of this set falls below $k$. Under \textbf{Case 3}, all models in the $\varepsilon'$ Rashomon set fall below $k$.}
        \label{fig:rld_proof_viz}
    \end{figure}
    
    We will consider three cases: first, we consider $\ell^* > k_1 \geq 0$, followed by $\varepsilon' + \ell^* > k_2 \geq \ell^*$, and finally $k_3 \geq \varepsilon' + \ell^*$. Figure \ref{fig:rld_proof_viz} provides a visual overview of these three cases and the broad idea within each case.

    \textbf{Case 1: $\ell^* > k_1 \geq 0$}
    
    For any $k_1$ such that $\ell^* > k_1 \geq 0$, it holds that 
    \begin{align*}
        \mathbb{1}[\ell(g^*, \Dn) \leq k_1] &= 0,
    \end{align*}
    since $\ell^* > k_1$ by definition. Further, because $\ell(g^*, \Dn) \leq \ell(f, \Dn)$ for mean squared error in the infinite data setting, 
    \begin{align*}
        \RLDCDF(k_1; \varepsilon', \mathcal{F}, \ell, \mathcal{P}_n, \lambda) = \RLDCDF(k_1; \varepsilon, \mathcal{F}, \ell, \mathcal{P}_n, \lambda) = 0
    \end{align*}
    
    \textbf{Case 2: $\varepsilon' + \ell^* \geq k_2 \geq \ell^*$}
    
    For any $k_2$ such that $\varepsilon' + \ell^* \geq k_2 \geq \ell^*$,
    \begin{align*}
        \mathbb{1}[\ell(g^*, \Dn) \leq k_2] &= 1,
    \end{align*}
    since $\ell(g^*, \Dn) \leq k_2$ by the definition of $k_2.$
    Let $\nu$ denote a volume function over the target model class. Recalling that $\varepsilon > \varepsilon'$, we know that:
    \begin{align*}
        \nu(\mathcal{R}^\varepsilon) >  \nu(\mathcal{R}^{\varepsilon'}) \iff  & \frac{1}{\nu(\mathcal{R}^\varepsilon)} <  \frac{1}{\nu(\mathcal{R}^{\varepsilon'})}\\
        \iff &  \frac{\nu(\{f \in \mathcal{R}^\varepsilon: \ell(f, \Dn) \leq k_2\})}{\nu(\mathcal{R}^\varepsilon)} <  \frac{\nu(\{f \in \mathcal{R}^\varepsilon: \ell(f, \Dn) \leq k_2\})}{\nu(\mathcal{R}^{\varepsilon'})}\\
        \iff &  \frac{\nu(\{f \in \mathcal{R}^\varepsilon: \ell(f, \Dn) \leq k_2\})}{\nu(\mathcal{R}^{\varepsilon})} <  \frac{\nu(\{f \in \mathcal{R}^{\varepsilon'}: \ell(f, \Dn) \leq k_2\})}{\nu(\mathcal{R}^{\varepsilon'})},
        \end{align*}
        since the set of models in the $\varepsilon$ Rashomon set with loss less than $k_2$ is the same set as set of models in the $\varepsilon'$ Rashomon set with loss less than $k_2$ for $k_2 \leq \varepsilon' + \ell^*.$ We can further manipulate this quantity to show:
        %
        \begin{align*}
        &\frac{\nu(\{f \in \mathcal{R}^\varepsilon: \ell(f, \Dn) \leq k_2\})}{\nu(\mathcal{R}^{\varepsilon})} <  \frac{\nu(\{f \in \mathcal{R}^{\varepsilon'}: \ell(f, \Dn) \leq k_2\})}{\nu(\mathcal{R}^{\varepsilon'})}\\
        \iff  & 1 - \frac{\nu(\{f \in \mathcal{R}^\varepsilon: \ell(f, \Dn) \leq k_2\})}{\nu(\mathcal{R}^{\varepsilon})} >  1 - \frac{\nu(\{f \in \mathcal{R}^{\varepsilon'}: \ell(f, \Dn) \leq k_2\})}{\nu(\mathcal{R}^{\varepsilon'})}\\
        \iff &  \left|1 - \frac{\nu(\{f \in \mathcal{R}^\varepsilon: \ell(f, \Dn) \leq k_2\})}{\nu(\mathcal{R}^{\varepsilon})}\right| >  \left|1 - \frac{\nu(\{f \in \mathcal{R}^{\varepsilon'}: \ell(f, \Dn) \leq k_2\})}{\nu(\mathcal{R}^{\varepsilon'})}\right| \\
        \iff &  \left|1 - \RLDCDF(k_2; \varepsilon, \mathcal{F}, \ell)\right| >  \left|1 - \RLDCDF(k_2; \varepsilon', \mathcal{F}, \ell)\right| \\
        \iff & \left|\LD(k_2) - \RLDCDF(k_2; \varepsilon, \mathcal{F}, \ell)\right| \\
        &>
        \left|\LD(k_2) - \RLDCDF(k_2; \varepsilon', \mathcal{F}, \ell) \right|,
    \end{align*}
    because $\LD(k_2) = 1$.
    
    \textbf{Case 3: $k_3 > \varepsilon' + \ell^*$}
    
    For any $k_3 > \varepsilon' + \ell^*$, we have 
    \begin{align*}
        \RLDCDF(k_3; \varepsilon', \mathcal{F}, \ell)  &= \frac{\nu(\{f \in \mathcal{R}^{\varepsilon'}: \ell(f, \Dn) \leq k_3\})}{\nu(\mathcal{R}^{\varepsilon'})}\\ 
        &= \frac{\nu(\mathcal{R}^{\varepsilon'})}{\nu(\mathcal{R}^{\varepsilon'})} & \text{because $k_3 > \varepsilon' + \ell^*$}\\
        &= 1.
    \end{align*}
    This immediately gives that
    \begin{align*}
        \left|\LD(k_3) - \RLDCDF(k_3; \varepsilon', \mathcal{F}, \ell)  \right| &= \left| 1 - 1 \right|\\&= 0,
    \end{align*}
    the minimum possible value for this quantity. We can then use the fact that the absolute value is greater than or equal to zero to show that
    \begin{align*}
         &\left|\LD(k_3) - \RLDCDF(k_3; \varepsilon, \mathcal{F}, \ell) \right| \\
         & \geq 0  =
        \left|\LD(k_3) - \RLDCDF(k_3; \varepsilon', \mathcal{F}, \ell)  \right| 
    \end{align*}
    
    In summary, under cases 1 and 3, 
    \begin{align*}
         &\left|\LD(k; \ell, n, \mathcal{P}_n,\lambda) - \RLDCDF(k; \varepsilon, \mathcal{F}, \ell, \mathcal{P}_n, \lambda) \right| \\
         &\geq
        \left|\LD(k; \ell, n, \mathcal{P}_n,\lambda) - \RLDCDF(k; \varepsilon', \mathcal{F}, \ell, \mathcal{P}_n, \lambda) \right|; 
    \end{align*}
    under case 2,
    \begin{align*}
         &\left|\LD(k_2; \ell, n, \mathcal{P}_n,\lambda) - \RLDCDF(k_2; \varepsilon, \mathcal{F}, \ell, \mathcal{P}_n, \lambda) \right| \\
         &>
        \left|\LD(k_2; \ell, n, \mathcal{P}_n,\lambda) - \RLDCDF(k_2; \varepsilon', \mathcal{F}, \ell, \mathcal{P}_n, \lambda) \right|.
    \end{align*}
    Since there is some range of values $k \in [\ell^*, \varepsilon' + \ell^*)$ for which the inequality above is strict, it follows that
    \begin{align*}
        &\int_{\ell_{\min}}^{\ell_{\max}}\left|\LD(k; \ell, n, \mathcal{P}_n,\lambda) - \RLDCDF(k; \varepsilon, \mathcal{F}, \ell, \mathcal{P}_n, \lambda) \right|dk \\
        >& \int_{\ell_{\min}}^{\ell_{\max}}\left|\LD(k; \ell, n, \mathcal{P}_n,\lambda) - \RLDCDF(k; \varepsilon', \mathcal{F}, \ell, \mathcal{P}_n, \lambda) \right|dk, \\
    \end{align*}
    showing that $\varepsilon > \varepsilon'$ is a \textit{sufficient} condition for $m(\varepsilon) > m(\varepsilon')$. Observe that, for a loss function with no regularization and a fixed model class, \textit{RLD} is a function of only $\varepsilon$. As such, varying $\varepsilon$ is the only way to vary \textit{RLD}, making $\varepsilon > \varepsilon'$ a \textit{necessary} condition for the above. Therefore, we have shown that $\varepsilon > \varepsilon' \iff m(\varepsilon) > m(\varepsilon')$, i.e. $m$ is strictly increasing. 
    
    Further, if $g^* \in \mathcal{F}$, the
    Rashomon set with $\varepsilon = 0$ will contain only $g^*$ as $n$ approaches infinity, immediately yielding that
    \begin{align*}
        m(0) = \int_{\ell_{\min}}^{\ell_{\max}}\left|\LD(k; \ell, n, \mathcal{P}_n,\lambda) - \RLDCDF(k; 0, \mathcal{F}, \ell) \right|dk = 0.
    \end{align*}
    
\end{proof}
Lemma \ref{lem:rld_convex} provides a mechanism  through which $\RLD$ will approach $\LD$ in the infinite data setting. The following lemma states that each level set of the quadratic loss surface is a hyper-ellipsoid, providing another useful tool for the propositions given in this section.
\begin{lemma}
\label{lem:ellipsoidal}
    The level set of the quadratic loss at $\varepsilon$ is a hyper-ellipsoid defined by:
    \begin{align*}
        (\theta - \theta^*)^TX^T X (\theta - \theta^*) = \varepsilon - c,
    \end{align*}
    which is centered at $\theta^*$ and of constant shape in terms of $\varepsilon$.
    
    \label{lemma}
\end{lemma}
\begin{proof}
    Recall that the quadratic loss for some parameter vector $\theta$ is given by:
    \begin{align*}
        \ell(\theta) &= \| y - X \theta \|^2
    \end{align*}
    and that the optimal vector $\theta^*$ is given by:
    \begin{align*}
        \theta^* &= (X^TX)^{-1}X^T y\\
        \iff X^TX\theta^* &= X^T y
    \end{align*}
    With these facts, we show that the level set for the quadratic loss at some fixed value $\varepsilon$ takes on the standard form for a hyper-ellipsoid. This is shown as:
    \begin{align*}
        \ell(\theta) &= \| y - X \theta \|^2\\
        &= \| y - X \theta \|^2 \underbrace{- y^T(y - X\theta^*) + y^T(y - X\theta^*)}_{\text{add }0 }\\
        &= \underbrace{y^Ty - 2y^T X \theta + \theta^T X^T X \theta}_{\text{expand quadratic}} \underbrace{- y^Ty - y^T X\theta^*}_{\text{distribute } y^T} + y^T(y - X\theta^*)\\
        &= y^Ty - 2y^T X \theta + \theta^T X^T X \theta - y^Ty - \underbrace{(X^Ty)^T \theta^*}_{\text{pull out transpose}} + y^T(y - X\theta^*)\\
        &= y^Ty - 2y^T X \theta + \theta^T X^T X \theta - y^Ty - \underbrace{\theta^{*T}X^TX\theta^*}_{\text{ because } X^Ty = X^TX\theta^*} + y^T(y - X\theta^*)\\
        &= y^Ty - \underbrace{2(X^Ty)^T \theta}_{\text{ pull out transpose }} + \theta^T X^T X \theta - y^Ty - \theta^{*T}X^TX\theta^* + y^T(y - X\theta^*)\\
        &= y^Ty - \underbrace{2\theta^*X^TX \theta}_{\text{ because } X^Ty = X^TX\theta^*} + \theta^T X^T X \theta - y^Ty - \theta^{*T}X^TX\theta^* + y^T(y - X\theta^*)\\
        &=  \theta^T X^T X \theta  - 2\theta^*X^TX \theta  - \theta^{*T}X^TX\theta^* + y^T(y - X\theta^*) \text{ because } y^Ty \text{ terms cancel out}\\
        &= (\theta - \theta^*)^TX^T X (\theta - \theta^*)+ y^T(y - X\theta^*) \text{ by factorization.}
    \end{align*}
    Noting that the term $y^T(y - X\theta^*)$ is constant in terms of $\theta$, so we can simplify this expression to $\ell(\theta) = (\theta - \theta^*)^TX^T X (\theta - \theta^*)+c$ where $c = y^T(y - X\theta^*)$. If we are interested in the level set at $\ell(\theta) = c + \varepsilon$ --- that is, with loss $\varepsilon$ greater than the optimal loss --- this is exactly:
    \begin{align*}
        &(\theta - \theta^*)^TX^T X (\theta - \theta^*)+c = c + \varepsilon\\
        \iff &(\theta - \theta^*)^TX^T X (\theta - \theta^*) = \varepsilon.
    \end{align*}
    That is, the set of parameters $\theta$ yielding loss value $c + \varepsilon$ is a hyper-ellipsoid centered at $\theta^*$ according to the positive semi-definite matrix $X^TX$.
\end{proof}

\begin{proposition}
\label{prop:linear}
    If the DGP is a linear regression model, Assumption \ref{asm:lipschitz_supp} is guaranteed to hold for the function class of linear models (i.e., $g^* \in \mathcal{F}$) as $n \to \infty$.
\end{proposition}
\begin{proof}
    We now turn our attention to \RID. Let our variable importance metric $\phi_j := \theta_j,$ the coefficient of a linear model, and let $p$ denote the number of variables in the dataset such that $\boldsymbol{\theta} \in \mathbb{R}^p$. As in Lemma \ref{lem:rld_convex}, we restrict ourselves to the setting in which $n \to \infty,$ although we often omit this notation.
    Define the function $r_j: [0, \ell_{\text{max}}] \to [0, 1]$ to be:
    \begin{align*}
        r_j(\varepsilon) := \int_{\phi_{\min}}^{\phi_{\max}} \left|\ourcdf_j(k; \{g^*\}, 0) - \ourcdf_j(k; \mathcal{F}, \varepsilon) \right| dk
    \end{align*}
    We show that $r_j$ is a monotonic function of $\varepsilon$, for any $j \in \{1, 2, \hdots, p\}$. In other words, as $\varepsilon$ gets smaller, the value of $r_j(\varepsilon)$ gets smaller. We do so by showing that the following holds for this VI metric: 
    \begin{align*}
        &\int_{\phi_{\min}}^{\phi_{\max}} \left|\ourcdf_j(k; \{g^*\}, 0) - \ourcdf_j(k; \mathcal{F}, \varepsilon) \right| dk \\
        \geq & \int_{\phi_{\min}}^{\phi_{\max}}\left|\ourcdf_j(k; \{g^*\}, 0) - \ourcdf_j(k; \mathcal{F}, \varepsilon') \right| dk
    \end{align*} 
    if and only if $\varepsilon > \varepsilon'$ by showing that, for any $k$, 
    \begin{align*}
        &\left|\ourcdf_j(k; \{g^*\}, 0) - \ourcdf_j(k; \mathcal{F}, \varepsilon)\right| \\
        &\geq 
        \left|\ourcdf_j(k; \{g^*\}, 0) - \ourcdf_j(k; \mathcal{F}, \varepsilon') \right|.
    \end{align*} 

    For simplicity of notation, we denote the linear regression model parameterized by some coefficient vector $\boldsymbol{\theta} \in \mathbb{R}^p$ simply as $\boldsymbol{\theta}$. Let $\boldsymbol{\theta}^* \in \mathbb{R}^p$ denote the coefficient vector for the optimal model. Additionally, we define the following quantities to represent the most extreme values for $\theta_j$ (i.e., the coefficient along the $j$-th axis) for each Rashomon set. Let $a_j$ and $b_j$ be the two values defined as:
    \begin{align*}
        a_j &:= \min_{\mathbf{v} \in \mathbb{R}^p} (\boldsymbol{\theta}^* + \mathbf{v})_j \text{ s.t. } \ell(\boldsymbol{\theta}^* + \mathbf{v}, \Dn) = \ell^* + \varepsilon\\
        b_j &:= \max_{\mathbf{v} \in \mathbb{R}^p}  (\boldsymbol{\theta}^* + \mathbf{v})_j \text{ s.t. } \ell(\boldsymbol{\theta}^* + \mathbf{v}, \Dn) = \ell^* + \varepsilon.
    \end{align*} 
     Similarly, let $a'_j$ and $b'_j$ be the two values defined as:
    \begin{align*}
        a'_j &:= \min_{\mathbf{v} \in \mathbb{R}^p}  (\boldsymbol{\theta}^* + \mathbf{v})_j \text{ s.t. } \ell(\boldsymbol{\theta}^* + \mathbf{v}, \Dn) = \ell^* + \varepsilon'\\
        b'_j &:= \max_{\mathbf{v} \in \mathbb{R}^p}  (\boldsymbol{\theta}^* + \mathbf{v})_j \text{ s.t. } \ell(\boldsymbol{\theta}^* + \mathbf{v}, \Dn) = \ell^* + \varepsilon'.
    \end{align*} 
    Intuitively, these  values represent the most extreme values of $\boldsymbol{\theta}$ along dimension $j$ that are still included in their respective Rashomon sets. Figure \ref{fig:a_b_explanation} provides a visual explanation of each of these quantities.
    \begin{figure}[t]
        \centering
        \includegraphics[width=0.9\textwidth]{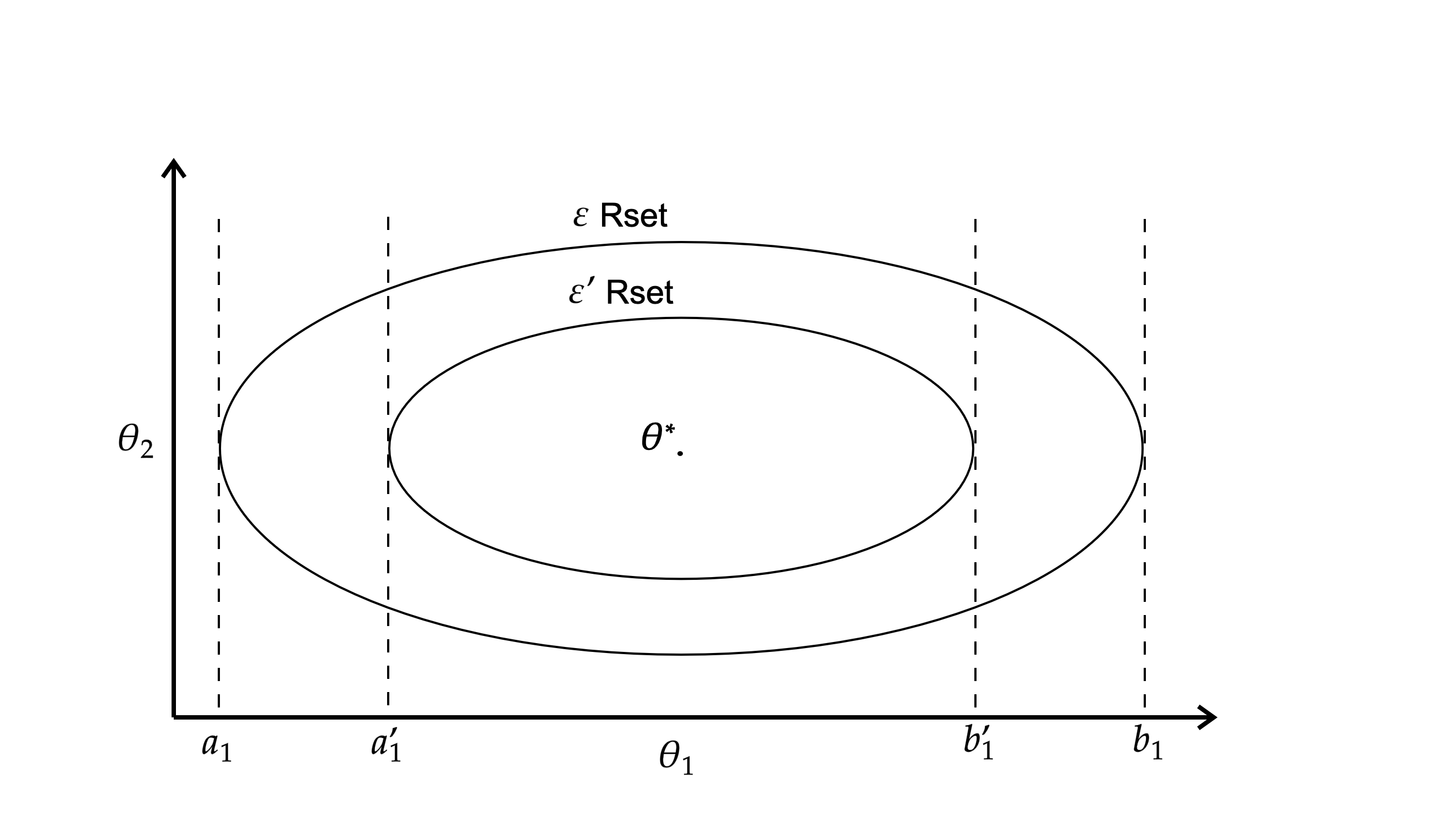}
        \caption{A visualization of the $\varepsilon$ and $\varepsilon'$ Rashomon sets for linear regression with two input features. We highlight the extrema of each Rashomon set along axis 1 ($a_1$ and $b_1$ for the $\varepsilon$ Rashomon set, $a'_1$ and $b'_1$ for the $\varepsilon'$ Rashomon set).}
        \label{fig:a_b_explanation}
    \end{figure}
    Finally, recall that:
    \begin{align*}
        \ourcdf_j(k; \{g^*\}, 0) &= \begin{cases}
            1 & \text{ if } \theta_j^* \leq k\\
            0 & \text{ otherwise}, 
        \end{cases}
    \end{align*}
    since $\boldsymbol{\theta}^*$ is a deterministic quantity given infinite data. 
    
    Without loss of generality, we will consider two cases: \begin{enumerate}
        \item The case where $\theta_j^* \leq k$,
        \item The case where $k < \theta_j^*$.
    \end{enumerate}
    Figures \ref{fig:case_1_viz} and \ref{fig:case_2_viz} give an intuitive overview of the mechanics of this proof. As depicted in Figure \ref{fig:case_1_viz}, we will show that the proportion of the volume of the $\varepsilon'$-Rashomon set with $\phi_j$ below $k$ is closer to 1 than that of the $\varepsilon$-Rashomon set under case 1. We will than show that the opposite holds under case 2, as depicted in Fugre \ref{fig:case_2_viz}.

    \subsubsection*{Case 1: $\theta_j^* \leq k$}
    \begin{figure}[t]
        \centering
        \includegraphics[width=0.9\textwidth]{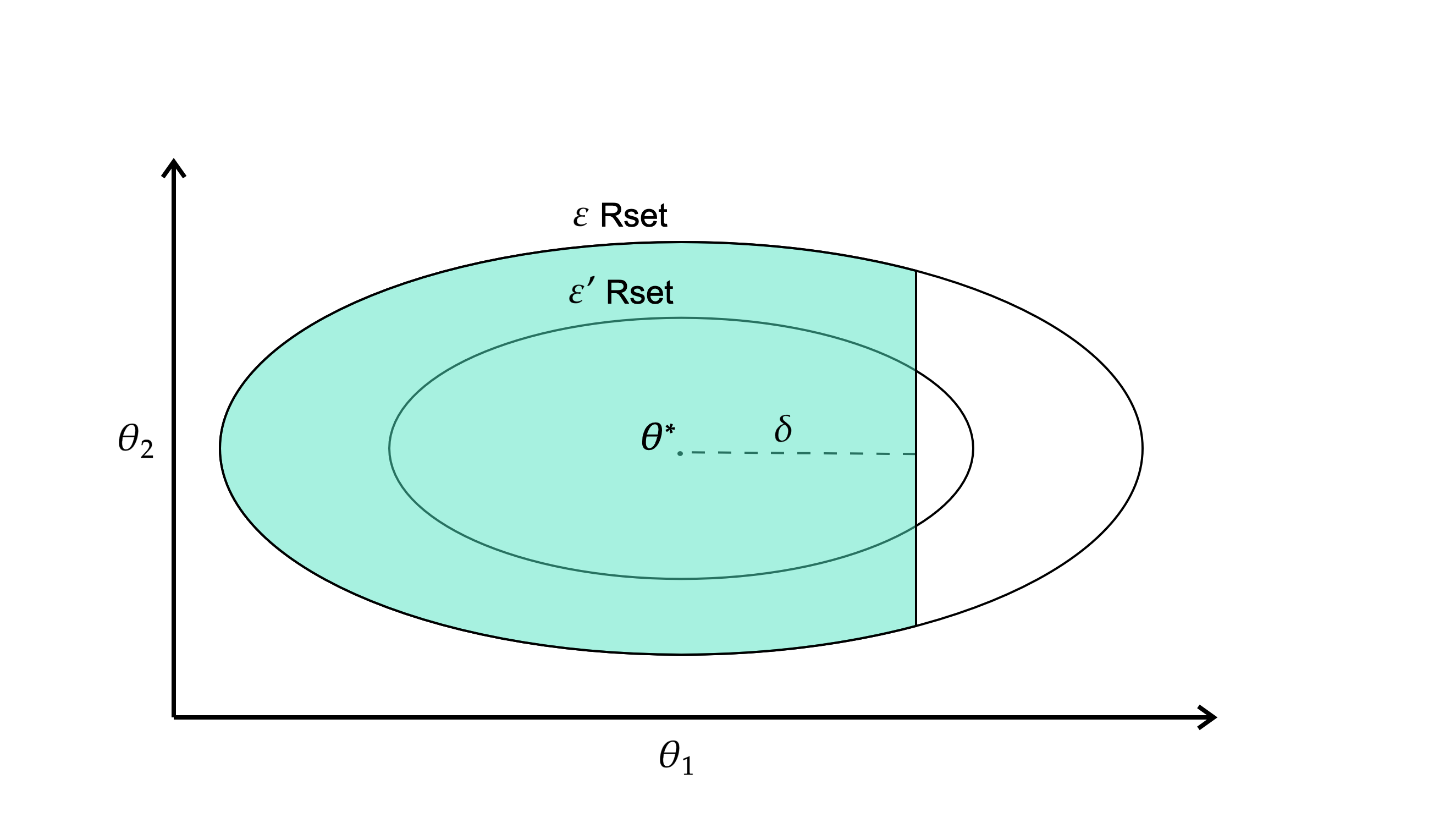}
        \caption{A simple illustration of the key idea in case 1 of the proof of Proposition \ref{prop:linear}. For two concentric ellipsoids of the same shape, the proportion of each ellipsoid's volume falling below some point greater than the center along axis $j$ is greater for the smaller ellipsoid than for the larger ellipsoid. }
        \label{fig:case_1_viz}
    \end{figure}
    
    Define two functions $h: [a_j, b_j] \rightarrow [0, 1]$ and $h': [a'_j, b'_j] \rightarrow [0, 1]$ as:
    \begin{align*}
        h(c) &=  \frac{c - a_j}{b_j - a_j}\\
        h'(c) &=  \frac{c - a'_j}{b'_j - a'_j}.
    \end{align*}
    These functions map each value $c$ in the original space of $\theta_j$ to its \textit{relative position} along each axis of the $\varepsilon$-Rashomon set and the $\varepsilon'$-Rashomon set respectively, with $h(b_j) = h'(b'_j) = 1$ and $h(a_j) = h'(a'_j) = 0$.

    Define $\delta \in [0, b_j - \theta_j^*]$ to be the value such that $k = \theta_j^* + \delta$. Since in this case $\theta^*_j \leq k$, it follows that $\delta \geq 0$. As such, we can then quantify the proportion of the $\varepsilon$-Rashomon set along the j-th axis such that $\theta_j^* \leq \theta_j \leq k$ as:
    \begin{align*}
        h(\theta_j^* + \delta) - h(\theta_j^*) &= \frac{(\theta^*_j + \delta) - a_j}{b_j - a_j} - \frac{(\theta_j^* - a_j)}{(b_j - a_j)} \\
        &= 
        \frac{\theta_j^* + \delta - a_j - \theta_j^* + a_j}{b_j - a_j}\\
        &= \frac{\delta}{b_j - a_j}
    \end{align*}
    Similarly, we can quantify the proportion of the $\varepsilon'$-Rashomon set along the $j$-th axis with $\theta_j$ between $k$ and $\theta_j^*$ as:
    \begin{align*}
        h'(\delta + \theta_j^*) - h'(\theta_j^*) &= \frac{\theta_j^* + \delta - a'_j - \theta_j^* + a'_j}{b'_j - a'_j}\\
        &= \frac{\delta}{b'_j - a'_j}.
    \end{align*}
    Recalling that, by definition, $a_j < a'_j < b'_j < b_j$, as well as the fact that $\delta \geq 0$ we can see that:
    \begin{align*}
        b_j - a_j > b'_j - a'_j &\iff \frac{1}{b_j - a_j} < \frac{1}{b'_j - a'_j}\\
        &\iff \frac{\delta}{b_j - a_j} \leq \frac{\delta}{b'_j - a'_j}\\
        &\iff h(\theta^*_j + \delta) - h(\theta^*_j) \leq h'(\theta^*_j + \delta) - h'(\theta^*_j) \\
        &\iff h(k) - h(\theta^*_j)  \leq h'(k) - h'(\theta^*_j) .
    \end{align*}
    That is, the proportion of the $\varepsilon$-Rashomon set along the $j$-th axis with $\theta_j$ between $k$ and $\theta^*_j$ is \textit{less than or equal to} the proportion of the $\varepsilon'$-Rashomon set along the $j$-th axis with $\theta_j$ between $k$ and $\theta^*_j$.  
    By Lemma \ref{lem:ellipsoidal}, recall that the $\varepsilon$-Rashomon set and the $\varepsilon'$-Rashomon set are concentric (centered at $\theta^*$) and similar (with shape defined by $X^TX$). Let $\nu$ denote the volume function for some subsection of a hyper-ellipsoid. 
    We then have 
    \begin{align*}
       h(k) - &h(\theta^*_j)  \leq h'(k) - h'(\theta^*_j)\\ 
       \iff &\frac{\nu(\{\boldsymbol{\theta} \in \mathcal{R}^\varepsilon : \theta^*_j \leq \theta_j \leq k\})}{\nu(\{\mathcal{R}^\varepsilon\})} \leq \frac{\nu(\{\boldsymbol{\theta}' \in \mathcal{R}^{\varepsilon'} : \theta^*_j \leq \theta'_j \leq k\})}{\nu(\{\mathcal{R}^{\varepsilon'}\})}\\
       \iff &\frac{1}{2} + \frac{\nu(\{\boldsymbol{\theta} \in \mathcal{R}^\varepsilon : \theta^*_j \leq \theta_j \leq k\})}{\nu(\{\mathcal{R}^\varepsilon\})} \leq \frac{1}{2} +  \frac{\nu(\{\boldsymbol{\theta}' \in \mathcal{R}^{\varepsilon'} : \theta^*_j \leq \theta'_j \leq k\})}{\nu(\{\mathcal{R}^{\varepsilon'}\})}\\
       \iff &\frac{\nu(\{\boldsymbol{\theta} \in \mathcal{R}^\varepsilon : \theta_j \leq \theta^*_j\})}{\nu(\{\mathcal{R}^\varepsilon\})} + \frac{\nu(\{\boldsymbol{\theta} \in \mathcal{R}^\varepsilon : \theta^*_j \leq \theta_j \leq k\})}{\nu(\{\mathcal{R}^\varepsilon\})}\\
       &\leq \frac{\nu(\{\boldsymbol{\theta}' \in \mathcal{R}^{\varepsilon'} : \theta'_j  \leq \theta^*_j\})}{\nu(\{\mathcal{R}^{\varepsilon'}\})} +  \frac{\nu(\{\boldsymbol{\theta}' \in \mathcal{R}^{\varepsilon'} : \theta^*_j \leq \theta'_j \leq k\})}{\nu(\{\mathcal{R}^{\varepsilon'}\})}\\
       \iff &\frac{\nu(\{\boldsymbol{\theta} \in \mathcal{R}^\varepsilon : \theta_j \leq k\})}{\nu(\{\mathcal{R}^\varepsilon\})} \leq \frac{\nu(\{\boldsymbol{\theta}' \in \mathcal{R}^{\varepsilon'} : \theta'_j \leq k\})}{\nu(\{\mathcal{R}^{\varepsilon'}\})}.
    \end{align*}
    Recalling that, by definition,  $\ourcdf_j(k; \mathcal{F}, \varepsilon') = \frac{\nu(\{\theta' \in \mathcal{R}^{\varepsilon'}: \theta'_j \leq k\})}{\nu(\mathcal{R}^{\varepsilon'})}$, it follows that:
    \begin{align*}
        &\ourcdf_j(k; \mathcal{F}, \varepsilon) \leq \ourcdf_j(k; \mathcal{F}, \varepsilon')\\
        &\iff  1 - \ourcdf_j(k; \mathcal{F}, \varepsilon) \geq 1 - \ourcdf_j(k; \mathcal{F}, \varepsilon')\\
        &\iff  |1 - \ourcdf_j(k; \mathcal{F}, \varepsilon)| \geq |1 - \ourcdf_j(k; \mathcal{F}, \varepsilon')|.
    \end{align*}
    Recalling that $\ourcdf_j(k; \{g^*\}, 0) = 1$, since $k \geq \theta^*_j$, the above gives:
    \begin{align*}
    |1 - &\ourcdf_j(k; \mathcal{F}, \varepsilon)| \geq |1 - \ourcdf_j(k; \mathcal{F}, \varepsilon')|\\
        \iff  &|\ourcdf_j(k; \{g^*\}, \varepsilon) - \ourcdf_j(k; \mathcal{F}, \varepsilon)|\\
        &\geq |\ourcdf_j(k; \{g^*\}, \varepsilon) - \ourcdf_j(k; \mathcal{F}, \varepsilon')|
    \end{align*}
    for all $\theta^*_j \leq k$.
    
    \subsubsection*{Case 2: $k < \mathbf{\theta}_j^* $}
    \begin{figure}[t]
        \centering
        \includegraphics[width=0.9\textwidth]{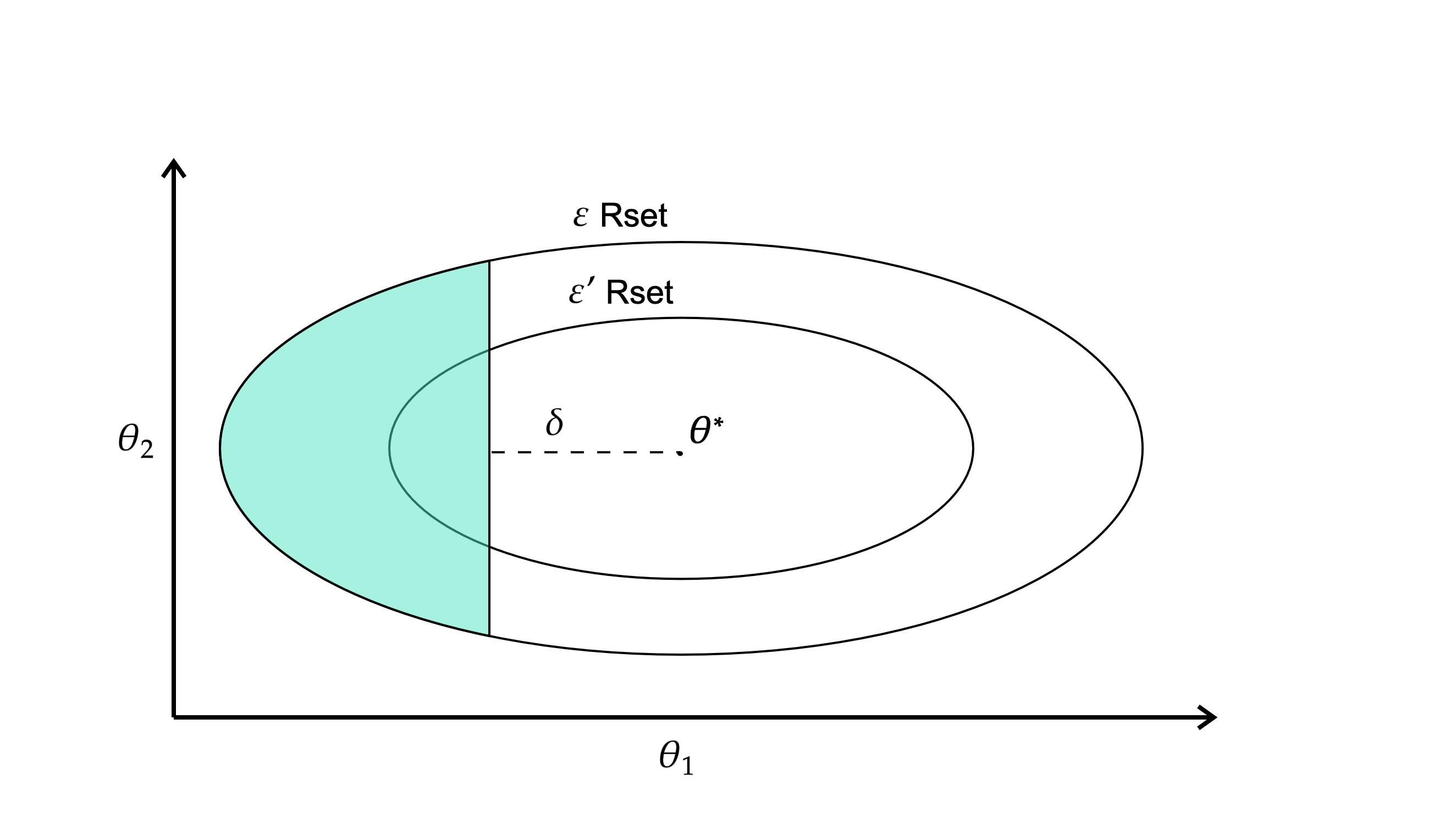}
        \caption{A simple illustration of the key idea in case 2 of the proof of Proposition \ref{prop:linear}. For two concentric ellipsoids of the same shape, the proportion of each ellipsoid's volume falling below some point less than the center along axis $j$ is smaller for the smaller ellipsoid than for the larger ellipsoid.}
        \label{fig:case_2_viz}
    \end{figure}
    Let $h$ and $h'$ be defined as in Case 1. Define $\delta \in [a_j - \theta^*_j, 0]$ to be the quantity such that $k = \theta^*_j + \delta$. In this case, $k < \theta^*_j$, so it follows that $\delta < 0$. Repeating the derivation from Case 1, we then have:
    \begin{align*}
        b_j - a_j > b'_j - a'_j &\iff \frac{1}{b_j - a_j} < \frac{1}{b'_j - a'_j}\\
        &\iff \frac{\delta}{b_j - a_j} > \frac{\delta}{b'_j - a'_j}\\
        &\iff h(\theta^*_j + \delta) - h(\theta^*_j) > h'(\theta^*_j + \delta) - h'(\theta^*_j) \\
        &\iff h(k) - h(\theta^*_j)  > h'(k) - h'(\theta^*_j) .
    \end{align*}
    That is, the proportion of the $\varepsilon$-Rashomon set along the $j$-th axis with $\theta_j$ between $k$ and $\theta^*_j$ is \textit{greater than} the proportion of the $\varepsilon'-$Rashomon set along the $j$-th axis with $\theta_j$ between $k$ and $\theta^*_j$.  By similar reasoning as in Case 1, it follows that:
    \begin{align*}
        &\ourcdf_j(k; \mathcal{F}, \varepsilon) > \ourcdf_j(k; \mathcal{F}, \varepsilon')\\
        &\iff  |\ourcdf_j(k; \mathcal{F}, \varepsilon) - 0| > |\ourcdf_j(k; \mathcal{F}, \varepsilon') - 0|
    \end{align*}
    Recalling that $\ourcdf_j(k; \{g^*\}, \varepsilon) = 0$, since $k < \theta^*_j$, the above gives:
    \begin{align*}
    |\mathbb{P}_{\Dn \sim \mathcal{P}_n}(\ourvi_j&(\mathcal{F}, \varepsilon) \leq k) - 0| > |\ourcdf_j(k; \mathcal{F}, \varepsilon') - 0|\\
        \iff  &|\ourcdf_j(k; \mathcal{F}, \varepsilon) - \ourcdf_j(k; \{g^*\}, 0)|\\
        &> |\ourcdf_j(k; \mathcal{F}, \varepsilon') - \ourcdf_j(k; \{g^*\}, 0)|
    \end{align*}
    for all $ a_j \leq k <\theta^*_j$. As such, for any $k$, we have that: 
    \begin{align*}
        &\left|\ourcdf_j(k; \{g^*\}, 0) - \ourcdf_j(k; \mathcal{F}, \varepsilon)\right| \\
        &\geq 
        \left|\ourcdf_j(k; \{g^*\}, 0) - \ourcdf_j(k; \mathcal{F}, \varepsilon') \right|,
    \end{align*}
    showing that $\varepsilon > \varepsilon'$ is a \textit{sufficient} condition for the above. Since \RID{} is a function of only $\varepsilon$, varying $\varepsilon$ is the only way to vary \RID{}, making $\varepsilon > \varepsilon'$ a \textit{necessary} condition for the above, yielding that $r_j(\varepsilon) > r_j(\varepsilon') \iff \varepsilon > \varepsilon'$ and $r_j$ is monotonically increasing.

    Let $m$ be defined as in Lemma \ref{lem:rld_convex}, and let $\gamma$ be some value such that $m(\varepsilon) \leq \gamma$. Define the function $d := r_j \circ m^{-1}$ (note that $m^{-1},$ the inverse of $m,$ is guaranteed to exist and be strictly increasing because $m$ is strictly increasing). The function $d$ is monotonically increasing as the composition of two monotonically increasing functions, and:
    \begin{align*}
        &m(\varepsilon) \leq \gamma \\
        \iff &\varepsilon \leq m^{-1}(\gamma)\\
        \iff & r_j(\varepsilon) \leq d(\gamma)\\
    \end{align*}
    as required. 
    
    Further, Lemma \ref{lem:rld_convex} states that $m(0) = 0$ if $g^* \in \mathcal{F}.$ Note also that the Rashomon set with $\varepsilon=0$ contains only $g^*$, and as such $r_j(0) = d(m^{-1}(0)) = 0$, meaning $d(0) = 0$. Therefore $\lim_{\gamma \to 0}d(\gamma) = 0$.

\end{proof}

\begin{proposition}
\label{prop:gam}
    Assume the DGP is a generalized additive model (GAM). Then, Assumption \ref{asm:lipschitz_supp} is guaranteed to hold for the function class of GAM's where our variable importance metric is the coefficient on each bin.
\end{proposition}
\begin{proof}
    Recall from Proposition \ref{prop:linear} that Assumption \ref{asm:lipschitz_supp} holds for the class of linear regression models with the model reliance metric $\phi_j = \theta_j$. A generalized additive model (GAM) \cite{hastie1990generalized} over $p$ variables is generally represented as:
    \begin{align*}
        g(\mathbb{E}[Y]) = \omega + f_1(x_1) + \hdots + f_p(x_p),
    \end{align*}
    where $g$ is some link function, $\omega$ is a bias term, and $f_1, \hdots, f_p$ denote the shape functions associated with each of the variables.
    In practice, each shape function $f_j$ generally takes the form of a linear function over binned variables \cite{lou2013accurate}:
    \begin{align*}
        f_j(x_i) = \sum_{j'=0}^{\beta_j - 1} \theta_{j'} \mathbb{1}[b_{j'} \leq x_{ij} \leq b_{j'+1}],
    \end{align*}
    where $\beta_j$ denotes the number of possible bins associated with variable $X_{j},$ $b_{j'}$ denotes the $j'$-th cuttoff point associated with $X_{j}$, and $\theta_{j'}$ denotes the weight associated with the $j'$-th bin on variable $X_j.$ With the above shape function, a GAM is a linear regression over a binned dataset; as such, for the variable importance metric $\phi_{j'} = \theta_{j'}$ on the complete, binned dataset, Assumption \ref{asm:lipschitz_supp} holds by the same reasoning as Proposition \ref{prop:linear}.
\end{proof}

%% file: supplement_files/detailed_experimental_setup.tex
In this work, we considered the following four simulation frameworks:
\begin{itemize}
    \item Chen's \cite{chen2017kernel}: $ Y = \mathbb{1}[-2\sin(X_1) + \max(X_2, 0) + X_3 + \exp(-X_4) + \varepsilon \geq 2.048],$ where $X_1, \ldots, X_{10}, \varepsilon \sim \mathcal{N}(0, 1).$ Here, only $X_1, \ldots, X_4$ are relevant.
    \item Friedman's \cite{friedman1991multivariate}: $Y = \mathbb{1}[ 10 \sin(\pi X_1 X_2) + 20 (X_3 - 0.5)^2 + 10X_4 + 5 X_5 + \varepsilon \geq 15],$ where $X_1, \ldots, X_{6} \sim \mathcal{U}(0, 1), \varepsilon \sim \mathcal{N}(0, 1).$ Here, only $X_1, \ldots, X_5$ are relevant.
    \item Monk 1 \cite{thrun1991monk}: $Y = \max\left(\mathbb{1}[X_1 = X_2], \mathbb{1}[X_5 = 1]\right),$ where the variables $X_1, \ldots, X_6$ have domains of 2, 3, or 4 unique integer values. Only $X_1, X_2, X_5$ are important.
    \item Monk 3 \cite{thrun1991monk}: $Y = \max\left( \mathbb{1}[X_5 = 3 \text{ and } X_4 = 1], \mathbb{1}[X_5 \neq 4 \text{ and } X_2 \neq 3] \right)$ for the same covariates in Monk 1. Here, $X_2, X_4,$ and $X_5$ are relevant, and $5\%$ label noise is added.
\end{itemize}

\begin{table}[h]
    \centering
    \begin{tabular}{c||c|c|c}
        DGP & Num Samples & Num Features & Num Extraneous Features \\
        \hline
        Chen's & 1,000 & 10 & 6\\
        Friedman's & 200 & 6 & 1\\
        HIV & 14,742 & 100 & Unknown\\
        Monk 1 & 124 & 6 & 3\\
        Monk 3 & 124 & 6 & 3\\
    \end{tabular}
    \caption{Overview of the size of each dataset considered (or generated from a DGP) in this paper.}
    \label{tab:dataset_desc}
\end{table}
For our experiments in Sections 4.1 and 4.2 of the main paper, we trained and evaluated all models using the standard training set provided by \cite{thrun1991monk} for Monk 1 and Monk 3. We generated 200 samples following the above process for Friedman's DGP, and 1000 samples following the above process for Chen's DGP.  

In Section 5 of the main paper, we evaluated \ourcdf on a dataset studying which host cell transcripts and chromatin patterns are associated with high expression of Human Immunodeficiency Virus (HIV) RNA. We used the model class of sparse decision trees and subtractive model reliance. The dataset combined single cell RNAseq/ATACseq profiles for 74,031 individual HIV infected cells from two different donors in the aims of finding new cellular cofactors for HIV expression that could be targeted to reactivate the latent HIV reservoir in people with HIV (PWH). A longer description of the data is in \cite{BrowneRedacted2023}. 

We consider the binary classification problem of predicting high versus low HIV load, where high HIV load means an HIV load in the top 10\% of observed values. We selected 14,614 samples (all 7,307 high HIV load samples and 7,307 random low HIV load samples) from the overall dataset in order to balance labels, and filtered the complete profiles down to the top 100 variables by individual AUC in order to accelerate the runtime of \ourcdf. 

Table \ref{tab:dataset_desc} summarizes the size of each dataset we considered. In all cases, we used random seed 0 for dataset generation, model training, and evaluation unless otherwise specified.

We compared the rankings produced by $\ourcdf$ with the following baseline methods: 
\begin{itemize}
    \item Subtractive model reliance $\phi^{\text{sub}}$ of a random forest (RF) \citep{breiman2001random} using scikit-learn's implementation \cite{scikit-learn} of RF
    \item Subtractive model reliance $\phi^{\text{sub}}$ of an L1 regularized logistic regression model (Lasso) using scikit-learn's implementation \cite{scikit-learn} of Lasso 
    \item Subtractive model reliance $\phi^{\text{sub}}$ of boosted decision trees \cite{freund1997decision} using scikit-learn's implementation \cite{scikit-learn} of AdaBoost
    \item Subtractive model reliance $\phi^{\text{sub}}$ of a generalized optimal sparse decision tree (GOSDT) \cite{lin2020generalized} using the implementation from \citep{xin2022exploring}
    \item Subtractive conditional model reliance (CMR) \cite{FisherRuDo19} -- a metric designed to capture only the unique information of a variable -- of RF using scikit-learn's implementation \cite{scikit-learn} of RF
    \item Subtractive conditional model reliance (CMR) \cite{FisherRuDo19} of Lasso using scikit-learn's implementation \cite{scikit-learn} of Lasso
    \item The impurity based model reliance metric for RF from \cite{breiman2001statistical} using scikit-learn's implementation \cite{scikit-learn} of RF 
    \item The LOCO algorithm reliance \cite{lei2018distribution} value for RF and for Lasso using scikit-learn's implementation \cite{scikit-learn} of both models
    \item The Pearson correlation between each feature and the outcome
    \item The Spearman correlation between each feature and the outcome  
    \item The mean of the partial dependency plot (PDP) \citep{greenwell2018simple} for each feature using scikit-learn's implementation \cite{scikit-learn}
    \item The SHAP value  \cite{lundberg2018consistent} for RF using scikit-learn's implementation \cite{scikit-learn} of RF 
    \item The mean of variable importance clouds (VIC) \cite{dong2020exploring} for the Rashomon set of sparse decision trees, computed using TreeFarms \cite{xin2022exploring}. 
\end{itemize}

We used the default parameters in scikit-learn's implementation \cite{scikit-learn} of each baseline model. The parameters used for \textit{RID}, VIC, and GOSDT for each dataset are summarized in Table \ref{tab:parameters}. In all cases, we constructed each of \textit{RID}, VIC, and GOSDT using the code from \citep{xin2022exploring}. 


\begin{table}[]
    \centering
    \begin{tabular}{c||c|c|c}
        Dataset & Rashomon Threshold $\varepsilon$ & Regularization Weight $\lambda$ & Depth Bound \\
        \hline
        Chen's & 0.01 & 0.01 & 5 \\
        Friedman's & 0.025 & 0.02 & 6 \\
        HIV & 0.075 & 0.005 & 3 \\
        Monk 1 & 0.1 & 0.03 & 5\\
        Monk 3 & 0.05 & 0.025 & 7 
    \end{tabular}
    \caption{The parameters used for \textit{RID}, VIC, and GOSDT by data generation process.}
    \label{tab:parameters}
\end{table}

\subsection{Computational Resources}
All experiments for this work were performed on an academic institution's cluster computer. We used up to 40 machines in parallel, selected from the specifications below:
\begin{itemize}
    \item 2 Dell R610's with 2 E5540 Xeon Processors (16 cores)
    \item 10 Dell R730's with 2 Intel Xeon E5-2640 Processors (40 cores)
    \item 10 Dell R610's with 2 E5640 Xeon Processors (16 cores)
    \item 10 Dell R620's with 2 Xeon(R) CPU E5-2695 v2's (48 cores)
    \item 8 Dell R610's with 2 E5540 Xeon Processors (16 cores)
\end{itemize}
We did not use GPU acceleration for this work.

%% file: supplement_files/additional_experiments.tex
\subsection{Recovering MR without Bootstrapping Baseline Methods}
\begin{figure}[h!]
    \centering
    \includegraphics[width=\textwidth]{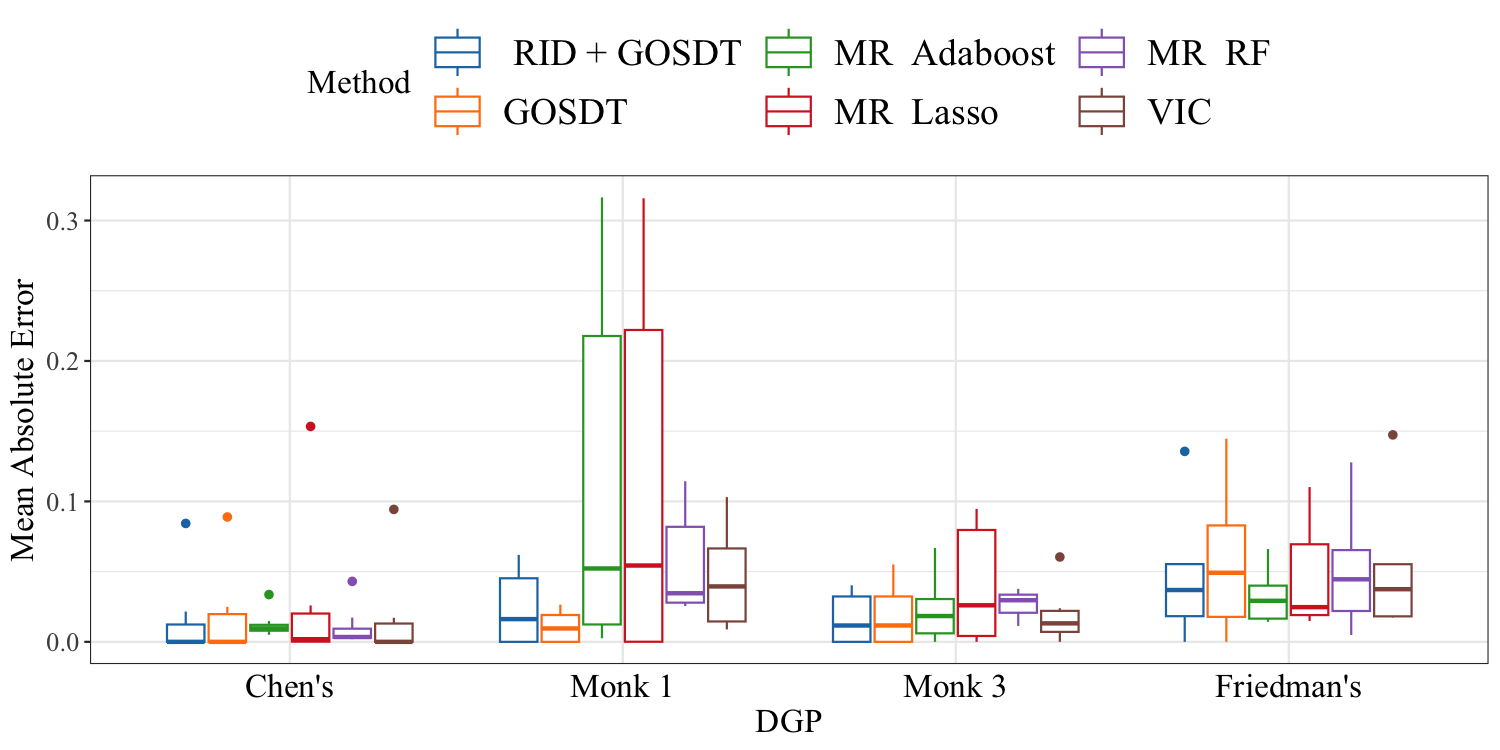}
    \caption{Boxplot over variables of the mean absolute error over test sets between the MR value produced by each method without bootstrapping (except \textit{RID}) and the model reliance of the DGP for 500 test sets.}
    \label{fig:scalar_mr_recovery}
\end{figure}

In this section, we evaluate the ability of each baseline method to recover the value of subtractive model reliance for the data generation process \textit{without bootstrapping}. 
%
For this comparison, we use one training set to find the model reliance of each variable for each of the following algorithms: GOSDT, AdaBoost, Lasso, and Random Forest. Because \ourcdf and VIC produce distributions/samples, we instead estimate the \textit{median} model reliance across \ourcdf and VIC's model reliance distributions. 

We then sample 500 test sets independently for each DGP. We then calculate the model reliance for each test set using the DGP as if it were a predictive model (that is, if the DGP were $Y = X + \varepsilon$ for some Gaussian noise $\varepsilon$, our predictive model would simply be $f(X) = X$). Finally, we calculate the mean absolute error between the test model reliance values for the DGP and the train model reliance values for each algorithm.

Figure \ref{fig:scalar_mr_recovery} shows the results of this experiment. As Figure \ref{fig:scalar_mr_recovery} illustrates, \textit{RID} produces more accurate point estimates than baseline methods even though this is not the goal of \ourcdf -- the goal of \RID{} is to produce \textit{the entire distribution} of model reliance across good models over bootstrap datasets, not a single point estimate. 

\subsection{Width of Box and Whisker Ranges}
When evaluating whether the box and whisker range (BWR) for each method captures the MR value for the DGP across test sets, a natural question is whether \ourcdf outperforms other methods simply because it produces wider BWR's. Figure \ref{fig:width_bwr} demonstrates the width of the BWR produced by each evaluated method across variables and datasets. As shown in Figure \ref{fig:width_bwr}, \textit{RID} consistently produces BWR widths on par with baseline methods.
\begin{figure}
    \centering
    \includegraphics[width=0.9\textwidth]{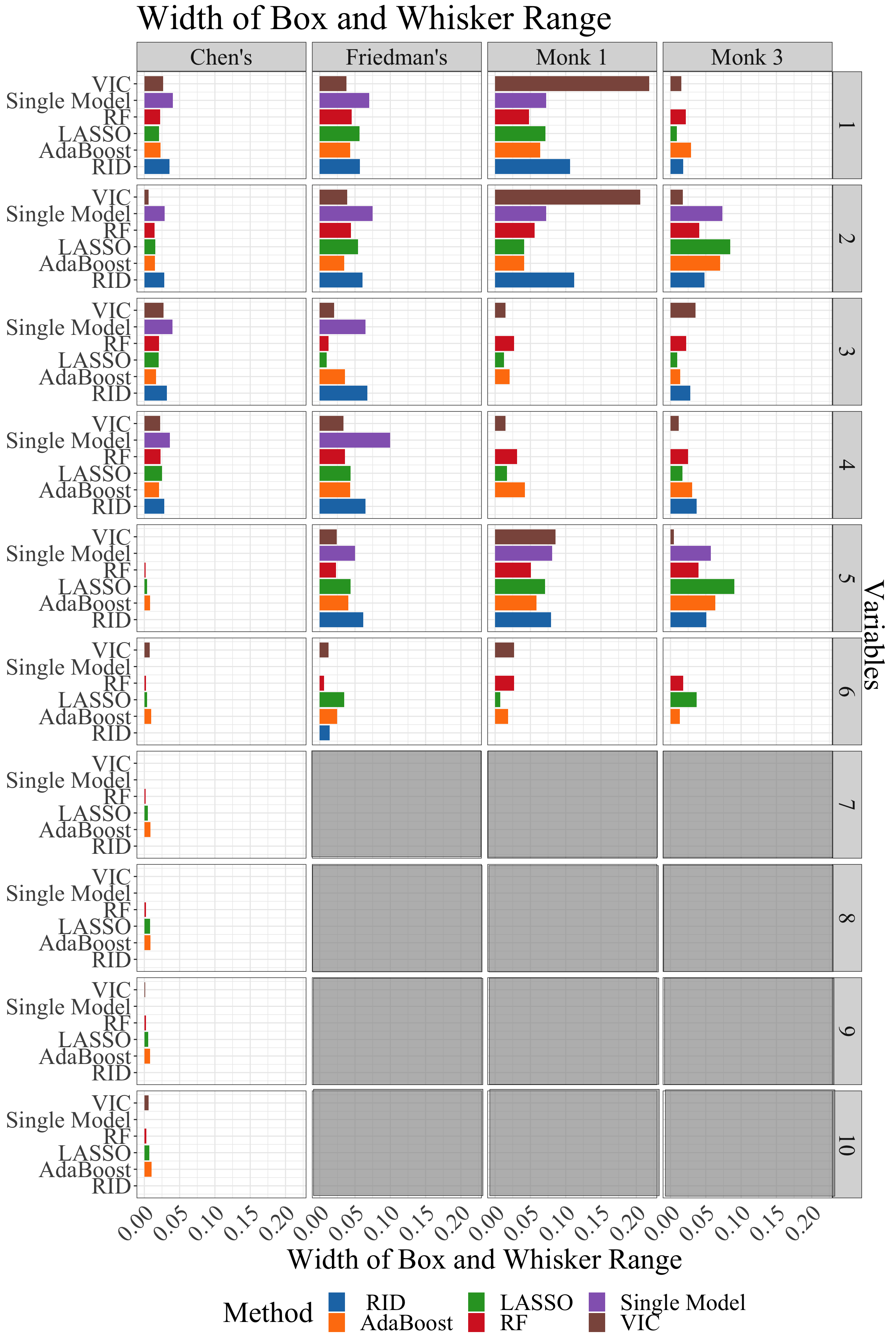}
    \caption{Width of the box and whisker range produced by each baseline method by dataset and variable. Gray subplots represent DGPs for which such a variable does not exist. Friedman's, Monk 1, and Monk 3 only have six variables.}
    \label{fig:width_bwr}
\end{figure}

\subsection{The Performance of \textit{RID} is Stable Across Reasonable Values for $\varepsilon$}
The parameter $\varepsilon$ controls what the maximum possible loss a model in the Rashomon set could be. We investigate whether this choice of $\varepsilon$ significantly alters the performance of \RID. In order to investigate this question, we repeat the coverage experiment from Section 4.2 of the main paper for three different values of $\varepsilon$ for each dataset on VIC and \ourcdf (the two methods effected by $\varepsilon$). In particular, we construct the BWR over 100 bootstrap iterations for \ourcdf and over models for VIC for three different values of $\varepsilon$ on each training dataset. These values are chosen as $0.75 \varepsilon^*$, $\varepsilon^*,$ and $1.25\varepsilon^*$, where $\varepsilon^*$ denotes the value of $\varepsilon$ used in the experiments presented in the main paper. We then generate 500 test datasets for each DGP and evaluate the subtractive model reliance for the DGP on each variable; we then measure what proportion of these test model reliance values are contained in each BWR. We refer to this proportion as the ``recovery percentage''.

Figure \ref{fig:epsilon_experiment} illustrates that \textbf{\textit{RID} is almost entirely invariant to reasonable choices of $\varepsilon$}: the recovery proportion for \textit{RID} ranges from $90.38\%$ to $90.64\%$ on Chen's DGP, $100\%$ to $100\%$ on Monk 1, $99.43\%$ to $99.93\%$ on Monk 3 DGP, and from $87.23\%$ to $88.8\%$ on Friedman's DGP. We find that VIC is somewhat more sensitive to choices of $\varepsilon$: the recovery proportion for VIC ranges from $83.44\%$ to $89.62\%$ on Chen's DGP, $100\%$ to $100\%$ on Monk 1, $75.30\%$ to $79.17\%$ on Monk 3 DGP, and from $60.53\%$ to $75.57\%$ on Friedman's DGP.

\begin{figure}[h]
    \centering
    \includegraphics[width=\textwidth]{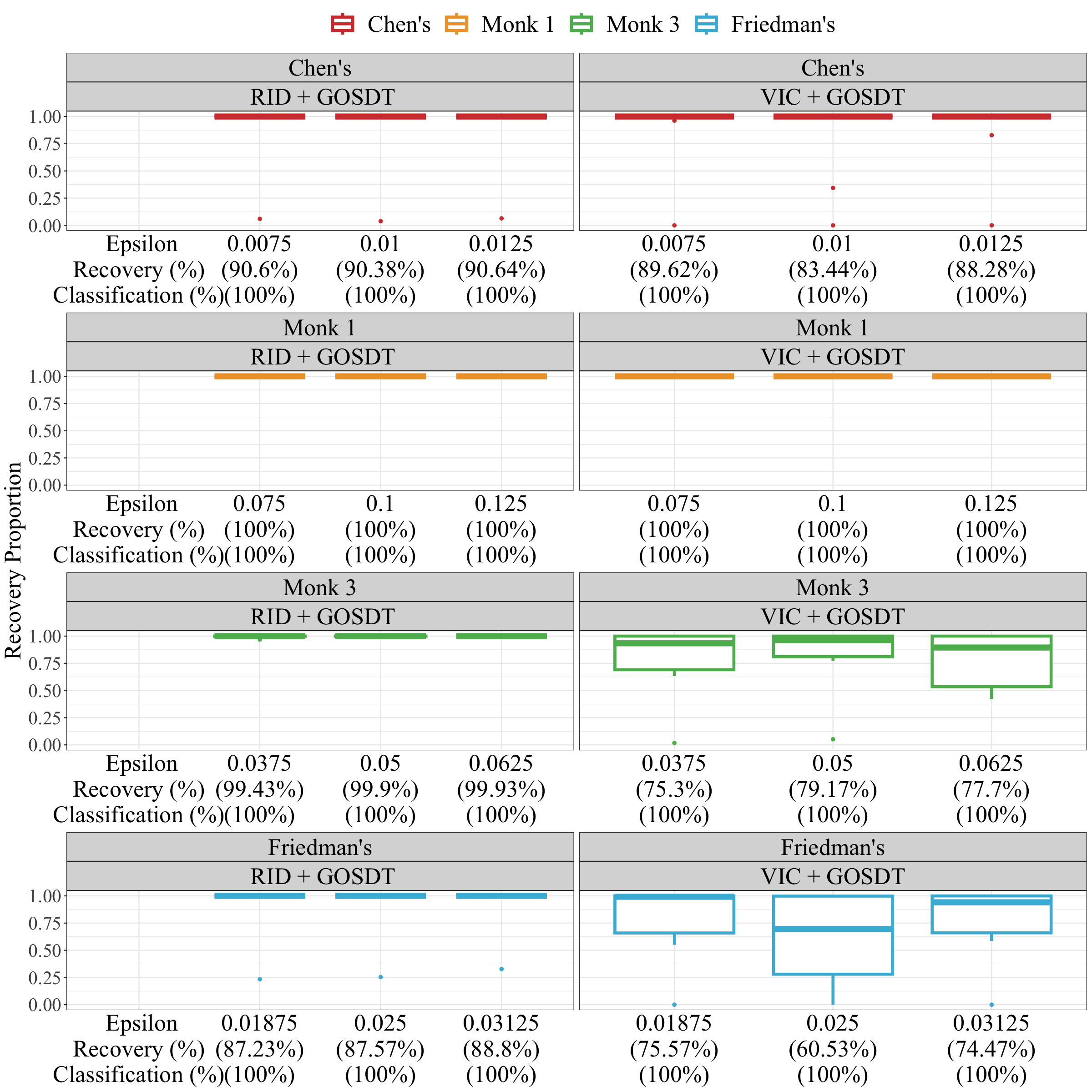}
    \caption{Box and whiskers plot over variables of the proportion test MR values for the DGP captured by the BWR range for $\ourcdf{}$ and VIC at different loss thresholds $\varepsilon$. We find that the performance of $\ourcdf{}$ is invariant to reasonable changes in $\varepsilon$.}
    \label{fig:epsilon_experiment}
\end{figure}

\subsection{Full Stability Results}
In this section, we demonstrate each interval produced by MCR, the BWR of VIC, and the BWR of $\ourcdf{}$ over 50 datasets generated from each DGP. We construct $\ourcdf{}$ using 50 bootstraps from each of the 50 generated datasets. 

Figures \ref{fig:chens_overlap}, \ref{fig:monk1_overlap}, \ref{fig:monk3_overlap}, and \ref{fig:friedmans_overlap} illustrate the 50 resulting intervals produced by each method for each non-extraneous variable on each DGP. If a method produces generalizable results, we would expect it to produce overlapping intervals across datasets drawn from the same DGP. As shown in Figures \ref{fig:chens_overlap}, \ref{fig:monk3_overlap}, and \ref{fig:friedmans_overlap}, both MCR and the BWR for VIC produced completely non-overlapping intervals between datasets for at least one variable on each of Chen's DGP, Monk 3, and Friedman's DGP, which means their results are not generalizable. In contrast, \textbf{the BWR range for \textit{$\ourcdf{}$ never} has zero overlap between the ranges produced for different datasets}. This highlights that $\ourcdf{}$ is more likely to generalize than existing Rashomon-based methods.

    
    \begin{figure}[h]
        \centering
        \includegraphics[width=0.82\textwidth]{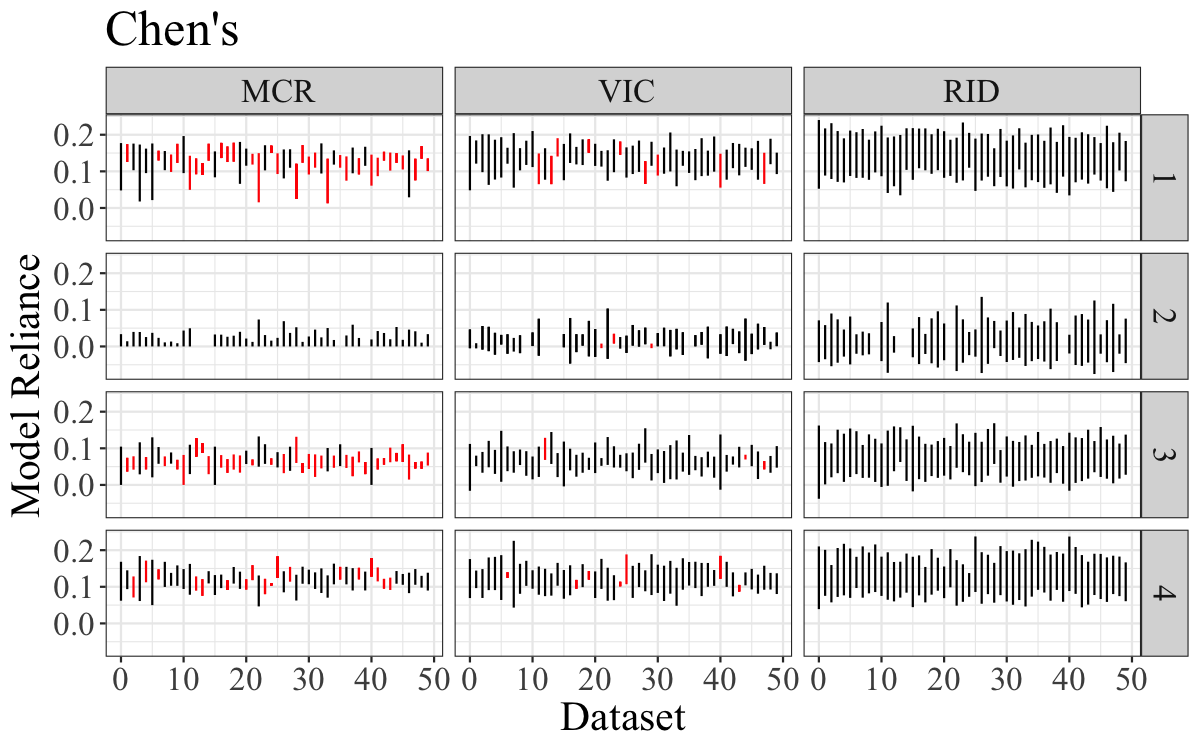}
        \caption{We generate 50 independent datasets from Chen's DGP and calculate MCR, BWRs for VIC, and BWRs for RID. 
        The above plot shows the interval for each dataset for each non-null variable in Chen's DGP. All \textcolor{red}{red}-colored intervals do not overlap with at least one of the remaining 49 intervals. }
        \label{fig:chens_overlap}
    \end{figure}
    
    \begin{figure}[h]
        \centering
        \includegraphics[width=0.82\textwidth]{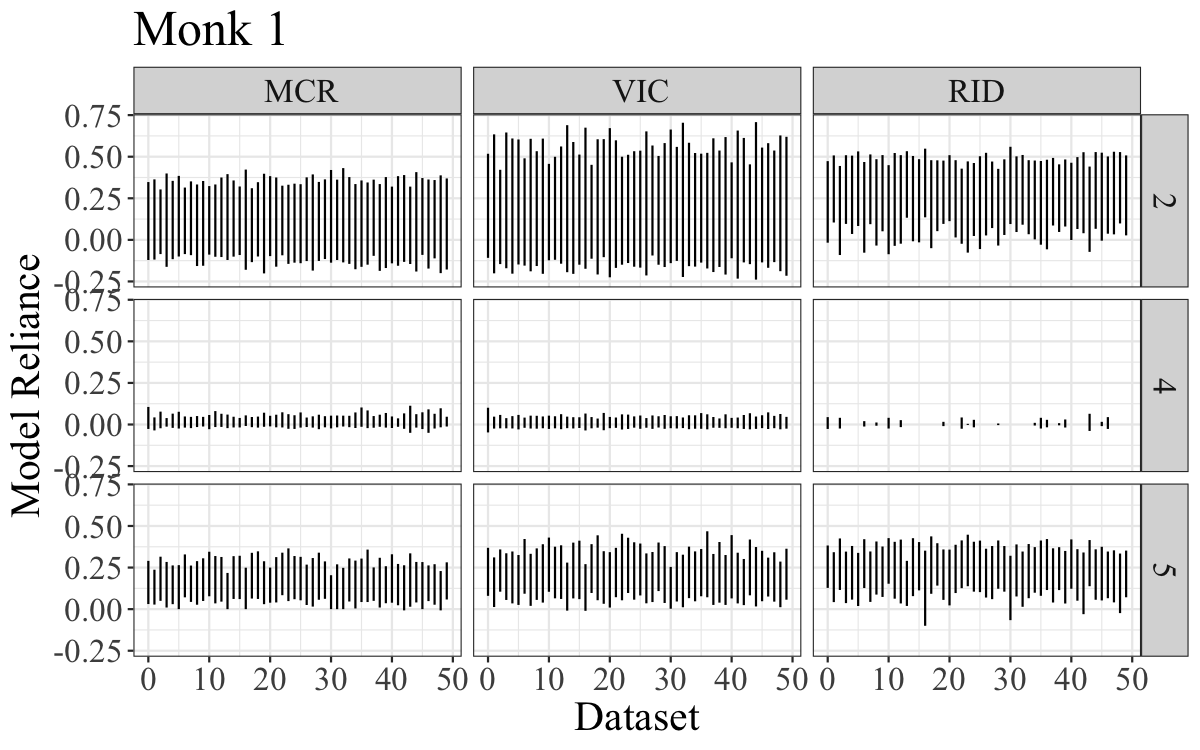}
        \caption{We generate 50 independent datasets from the Monk 1 DGP and calculate MCR, BWRs for VIC, and BWRs for RID. 
        The above plot shows the interval for each dataset for each non-null variable in Monk 1 DGP. All \textcolor{red}{red}-colored intervals (there are none in this plot) do not overlap with at least one of the remaining 49 intervals. }
        \label{fig:monk1_overlap}
    \end{figure}

    \begin{figure}[h]
        \centering
        \includegraphics[width=0.82\textwidth]{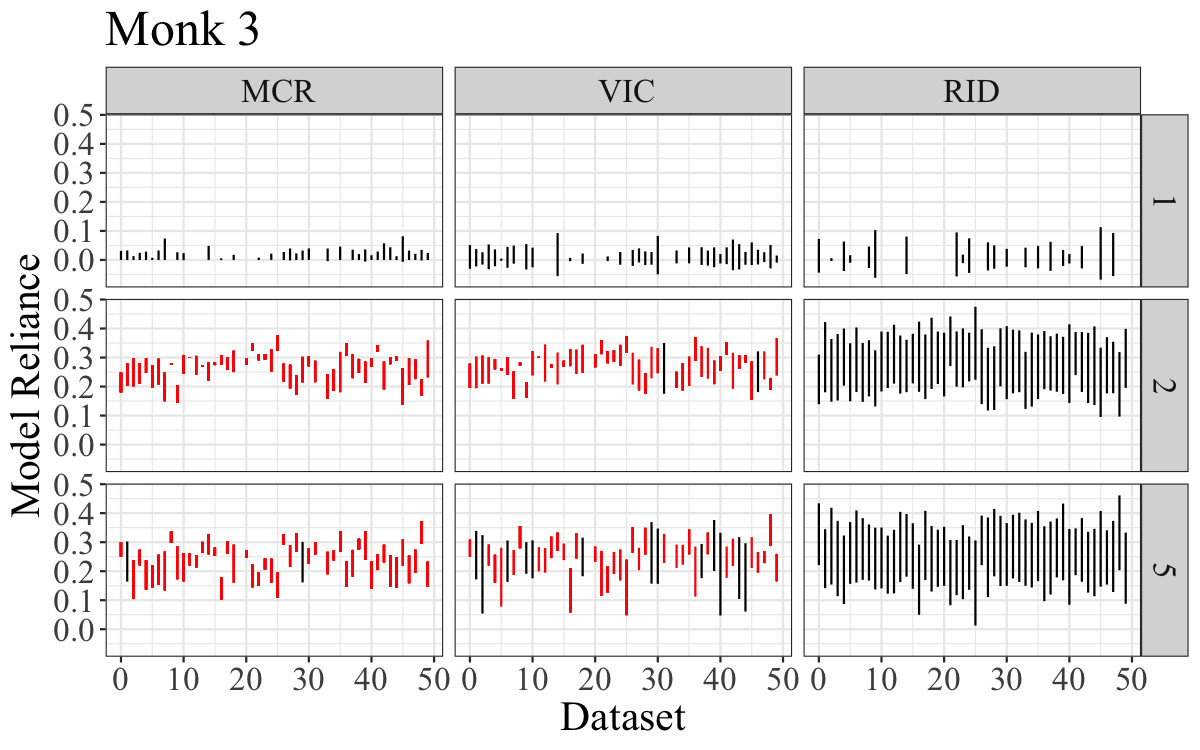}
        \caption{We generate 50 independent datasets from the Monk 3 DGP and calculate MCR, BWRs for VIC, and BWRs for RID. 
        The above plot shows the interval for each dataset for each non-null variable in the Monk 3 DGP. All \textcolor{red}{red}-colored intervals do not overlap with at least one of the remaining 49 intervals. }
        \label{fig:monk3_overlap}
    \end{figure}

    \begin{figure}[h]
        \centering
        \includegraphics[width=0.82\textwidth]{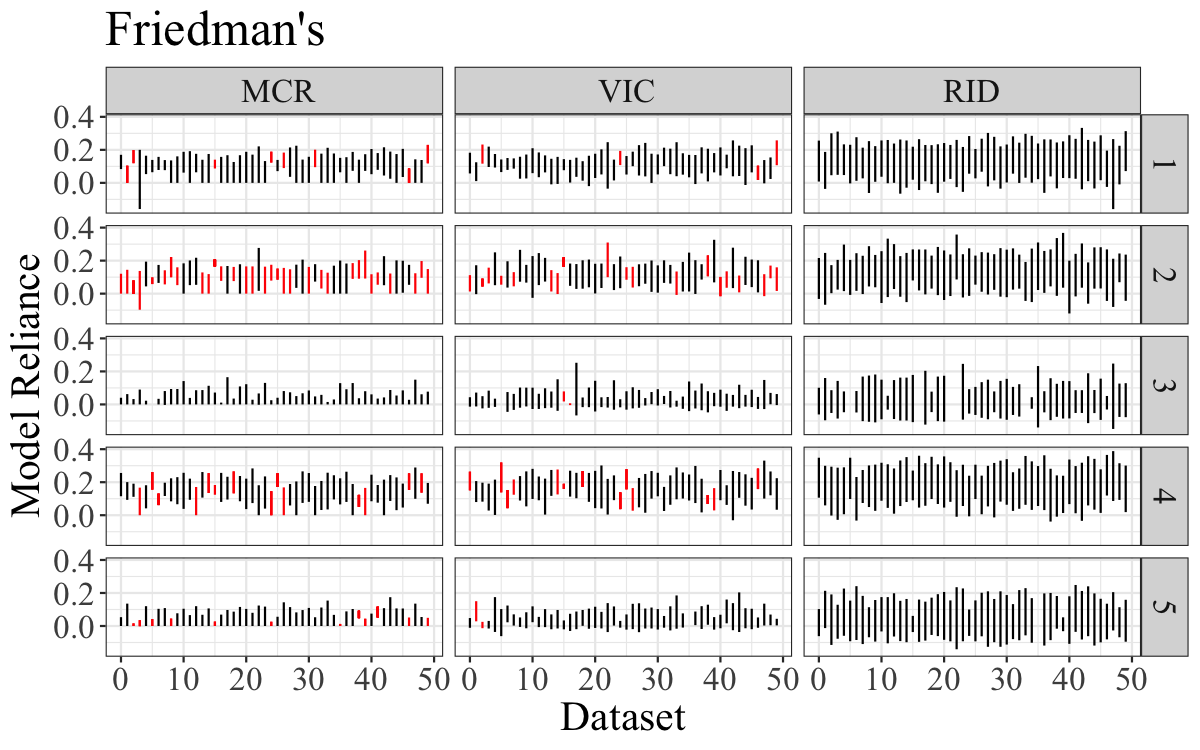}
        \caption{We generate 50 independent datasets from Friedman\'s DGP and calculate MCR, BWRs for VIC, and BWRs for RID. 
        The above plot shows the interval for each dataset for each non-null variable in Friedman's DGP. All \textcolor{red}{red}-colored intervals do not overlap with at least one of the remaining 49 intervals. }
        \label{fig:friedmans_overlap}
    \end{figure}

\subsection{Timing Experiments}
Finally, we perform an experiment studying how well the runtime of $\ourcdf$ scales with respect to the number of samples and the number of features in the input dataset using the HIV dataset \cite{BrowneRedacted2023}. The complete dataset used for the main paper consists of 14,742 samples measuring 100 features each. We compute $\ourcdf$ using 30 bootstrap iterations for each combination of the following sample and feature subset sizes: 14,742 samples, 7,371 samples, and 3,686 samples; 100 features, 50 features, and 25 features. 

Note that, in our implementation of $\ourcdf,$ any number of bootstrap datasets may be handled in parallel; as such, we report the mean runtime per bootstrap iteration in Table \ref{tab:timing}, as this quantity is independent of how many machines are in use. As shown in Table \ref{tab:timing}, $\ourcdf$ scales fairly well in the number of samples included, and somewhat less well in the number of features. This is because the number of possible decision trees grows rapidly with the number of input features, making finding the Rashomon set a more difficult problem and leading to larger Rashomon sets. Nonetheless, even for a large number of samples and features, $\ourcdf$ can be computed in a tractable amount of time: with 100 features and 14,742 samples, we found an average time per bootstrap of about
52 minutes.
\begin{table}[]
    \centering
    \begin{tabular}{c||c|c|c}
        \diagbox[width=8em]{Samples}{Variables} & 25 & 50 & 100 \\
        \hline
        3,686 & 19.3 (0.9) & 64.2 (6.2) & 164.0 (14.6)\\
        7,371 & 40.5 (2.5) & 177.7 (18.8)  & 723.1 (106.4)\\
        14,742 & 92.9 (6.8) & 431.4 (39.9) & 3128.7 (281.9)\\
    \end{tabular}
    \caption{Average runtime in seconds per bootstrap for $\ourcdf$ as a function of the number of variables and number of samples included from the HIV dataset. The standard error about each average is reported in parentheses.}
    \label{tab:timing}
\end{table}
